\def\year{2021}\relax
\documentclass[letterpaper]{article} 
\usepackage{aaai21}  
\usepackage{times}  
\usepackage{helvet} 
\usepackage{courier}  
\usepackage[hyphens]{url}  
\usepackage{graphicx} 
\usepackage{graphicx}  
\usepackage{natbib}  
\usepackage{caption} 
\usepackage[switch]{lineno}

\title{
A Recipe for Global Convergence Guarantee in Deep Neural Networks
}

\author{
        \hspace{-16pt}
        Kenji Kawaguchi\textsuperscript{*} \hspace{95pt}
        Qingyun Sun\textsuperscript{*}\\
}
\affiliations{
    {\fontsize{12pt}{12pt}\selectfont Harvard University} \hspace{85pt} 
    {\fontsize{12pt}{12pt}\selectfont Stanford University} \\
    \hspace{-18pt}
    \texttt{kkawaguchi@fas.harvard.edu}  \hspace{40pt}
    \texttt{qysun@stanford.edu}
}

\usepackage{subcaption}
\usepackage{mathtools}
\usepackage{bm}
\usepackage{algorithm}
\usepackage[noend]{algorithmic}
\usepackage{booktabs}
\usepackage{slashbox}
\usepackage{todonotes}
\usepackage{amsmath}
\usepackage{Definitions}
\usepackage{color}

\newcommand\blfootnote[1]{%
  \begingroup
  \renewcommand\thefootnote{}\footnote{#1}%
  \addtocounter{footnote}{-1}%
  \endgroup
}

\begin{document}


\maketitle

\begin{abstract}
Existing global convergence guarantees of (stochastic) gradient descent do not apply to practical deep networks in the practical regime of deep learning beyond the neural tangent kernel (NTK) regime. This paper proposes an algorithm, which is ensured to have global convergence guarantees in the practical regime beyond the NTK regime, under a verifiable condition called the expressivity condition. The expressivity condition is defined to be both data-dependent and architecture-dependent, which is the key property that makes our results applicable for practical settings beyond the NTK regime. On the one hand, the expressivity condition is theoretically proven to hold data-independently for fully-connected deep neural networks with narrow hidden layers and a single wide layer. On the other hand, the expressivity condition is numerically shown to hold data-dependently for deep (convolutional) ResNet with batch normalization with various standard image datasets. We also show that the proposed algorithm has generalization performances comparable with those of the heuristic algorithm, with the same hyper-parameters and total number of iterations. Therefore, the proposed algorithm can be viewed as a step towards providing theoretical guarantees for deep learning in the practical regime.
\end{abstract}

\section{Introduction}
The pursuit of global convergence guarantee has been one of the important aspects of optimization theory. However, ensuring global convergence is notoriously hard for first-order optimization algorithms used to train deep neural networks \citep{goodfellow2016deep}. Recently, some progress has been made on understanding the optimization aspect of overparametrized neural networks. Overparametrized neural networks can be trained to have zero training errors, interpolating all the training data points, and are recently shown to have global convergence guarantees in theoretical regimes  \citep{li2018learning,soltanolkotabi2018theoretical,kawaguchi2019gradient,daniely2019neural,bresler2020corrective,montanari2020interpolation,bubeck2020network}. These studies open up an insightful direction leading to the understanding of the optimization aspect of deep learning. 

However, there is still a significant gap between theory and practice. In applications such as computer vision, speech and natural language, a major reason for the success of deep learning in practice is its ability to learn representations with multiple levels of abstraction during training, as explained by  \citet{lecun2015deep}. In contrast, special types of neural networks studied in previous theories with global convergence guarantees are not allowed to learn representation during training, as the neural tangent kernels are approximately unchanged during training. Indeed, such special neural networks without the capability to learn representation are considered to have limitations compared to those with the capability  \cite{wei2019regularization,chizat2019lazy,yehudai2019power}.
Furthermore, the set of neural networks studied by previous theories have not yet practical deep neural networks used in practice with good generalization performances \citep{kawaguchi2017generalization, poggio2017theory}.  \blfootnote{*Equal contribution}

In this work, we propose a two-phase method to modify a base algorithm such that the modified algorithm enables practical deep neural networks to learn representation while having global convergence guarantees \textit{of all layers} under verifiable conditions. Our global convergence guarantees are applicable to a wide range of practical deep neural networks, including deep convolutional networks with  skip connection and batch normalization. For example, the verifiable conditions for global convergence guarantees are shown to be satisfied by both fully connected deep neural networks and deep  residual neural networks (ResNets) with convolutional layers. Our main  contributions can be summarized as:
\begin{itemize}
\item 
We propose a novel algorithm that  turns  any given first-order  training algorithm into a two-phase training algorithm.
\item
We prove that the resulting two-phase training algorithms find global minima for all layers of deep neural networks,    under the  \textit{expressivity condition}.
\item
The condition for global convergence is verified theoretically for fully connected networks with  last hidden layer being wide (as the number of training data points) and all other hidden layers being  narrow (as the input dimension).
\item
The condition for global convergence is verified numerically for the deep (convolutional) ResNet with bath-normalization on various standard  datasets.
\item
We compare the standard   training algorithm (SGD with momentum) and  the two-phase version of it with the same hyperparameters and total iterations. The two-phase version is shown to preserve the practical  generalization performances of the standard training  while providing global convergence guarantees.   
\end{itemize}

\section{Related work}
In this section, we discuss related studies  and their relationships with the contributions  of this paper.

\subsubsection{Over-parameterization} Over-parameterization has been shown to  help optimization of neural networks. More concretely, over-parameterization can remove suboptimal local minima \citep{soudry2016no} and  improve the quality of random initialization \citep{safran2016quality}. Furthermore, gradual over-parameterization (i.e., gradually increasing the number of parameters) is recently shown to improve steadily the quality of local minima \citep{kawaguchi2019effect}. The extreme over-parameterization that requires the number of neurons to approach  infinity is   used to prove global convergence \citep{mei2018mean,mei2019mean,chizat2018global,dou2020training,wei2019regularization,fang2020modeling}. Polynomial degrees of over-parameterization are also utilized for global convergence in the lazy training regime.

\subsubsection{Neural tangent kernel and lazy training}
It was shown that  neural networks under lazy training regime  (with a specific scaling and initialization) is nearly a linear model fitted with random features induced by the neural tangent kernel (NTK) at random initialization. Accordingly, in the lazy training regime, which is also called the NTK regime,  neural networks provably achieve globally minimum training errors. The lazy training regime is studied for both shallow (with one hidden layer) and deep neural networks and convolutional networks in  previous studies  \cite{zou2020gradient,li2018learning,jacot2018neural,du2019gradient,du2018gradient,chizat2019lazy,arora2019fine,allen2019learning,fang2020modeling,montanari2020interpolation}. 
\subsubsection{Lazy training and degree of overparametrization}
The global convergence guarantee in the lazy training regime was first proven  by using the significant overparametrization that requires the number of neurons per layer to be large polynomials in the number of data points  \cite{li2018learning,soltanolkotabi2018theoretical}. Later, the requirement on the degree of over-parametrization has been  improved to a small polynomial dependency  \cite{kawaguchi2019gradient,bresler2020corrective}. Furthermore, for two-layer networks with random i.i.d. weights and i.i.d. input data, the requirement was reduced to the number of training data points divided by the input dimension up to log factors, which is the optimal order in theory \citep{daniely2019neural,montanari2020interpolation,bubeck2020network}.

\subsubsection{Beyond lazy training regime}
However, it has been noted that neural networks in many real-world applications has weight parameters trained beyond the lazy training  regime, so that the learned features have better expressive power than random features \cite{yehudai2019power,misiakiewicz2019limitations,arora2019fine,arora2019exact}. Accordingly, a series of studies have demonstrated that the lazy training perspective of neural networks is not enough for understanding the success of deep learning  \cite{wei2019regularization,chizat2019lazy,yehudai2019power}. Indeed, there are also previous works for the regime beyond the lazy training \citep{kawaguchi2016deep, kawaguchi2019depth, jagtap2020adaptive, jagtap2020locally}.  
To overcome the weakness of lazy training, in this work, we present a novel method to use learned representation with a learned neural tangent kernel, instead of standard lazy training that use almost the random initialized neural tangent kernel. Our experiments on multiple ML benchmark datasets show empirically that our two-phase training method achieves comparable generalization performances with standard SGD training.

\subsubsection{Relation to this paper} 
Unlike previous work on the lazy training regime that use the NTK at random initialization, we allow the NTK to change significantly during training, to learn features and representation. In terms of the degree of overparametrization, the results in this paper achieve the linear order (in the number of training data points) without  the assumptions of  the i.i.d. weights and i.i.d random input. Our results  are also applicable for  deep neural networks in practical settings without degrading the generalization performances. On the other hand,  this paper further shows that the study of lazy training regime is also useful to understand the new two-phase training algorithm. Thus, we hope that the proposed two-phase training algorithm  becomes a bridge between practice and theory of neural tangent kernel.

\section{Model}
In this paper, we consider the empirical risk minimization problem. 
Let $((x_i,y_i))_{i=1}^n$ be a training dataset of  $n$ samples where  $S_{x}=\{x_i\}_{i=1}^n$ and $S_{y}=\{y_i\}_{i=1}^n$ are the set of training inputs and target outputs, with $x_i \in\Xcal \subseteq  \RR^{m_x}$  and $y_i \in\Ycal \subseteq  \RR^{m_y}$.  
Let $\ell: \RR^{m_{y}} \times\Ycal \rightarrow \RR_{\ge 0}$ be the loss of each sample that measures the difference between the prediction $f(x_{i},w)$ and the target $y_{i}$. 
The goal of empirical risk minimization is to find a prediction function $f(\hspace{1pt}\cdot\hspace{2pt};w): \RR^{m_x}  \rightarrow \RR^{1 \times m_y}$, by minimizing 
$$
\Lcal(w)= \frac{1}{n} \sum_{i=1}^n \ell(f(x_{i},w)\T,y_{i})
$$
where
 $w \in \RR^d$ is the parameter vector that contains all the trainable parameters, including the weights and bias terms of all layers of  deep neural networks. We define $w_{(l)}\in \RR^{d_{l}}$ to  be the  vector of all the trainable parameters at $l$-th layer. For any pair $(r,q)$ such that $1\le r \le q \le H+1$, let   $w_{(r:q)}=[w_{(r)}\T,\dots,w_{(q)}\T]\T  \in \RR^{d_{r:q}}$ 
where $w_{(r:q)}=w_{(r)}$ when $r=q$. With this notation, we  can write $w=w_{(1:H+1)}$.

 Here, $f(x, w)$ represents the \textit{pre-activation} output of the last layer of a neural network for a given $(x,w)$. Then the output over all data points is $f_X(w)=[f(x_{1}, w)\T, \dots,f(x_{n}, w)\T]\T \in \RR^{n \times m_y}$. The pre-activation of the last layer is an affine map given by  
$$
f_{X}(w)= \mathbf{h_X^{(H)}}W^{(H+1)}+b^{(H+1)},
$$ 
where  $W^{(H+1)}\in \RR^{ m_{H} \times m_{y}}$ and $b^{(H+1)}\in \RR^{1 \times m_{y}}$ are the weight matrix and the bias term   at the last layer. Here,   
 $$
 \mathbf{h_X^{(H)}} = h^{(H)}_X(w_{(1:H)})\in \RR^{n \times m_H}
 $$ is the matrix that contains the outputs at the last hidden layer.
In order to consider  layers with batch normalization, we allow  $h^{(H)}_X(w_{(1:H)})$ and $f(x_{i}, w)$ for each  $i \in \{1,\dots,n\}$ to depend on all data points $x_{1},\dots,x_n$. Here, $w_{(1:H)}$  represents the vector of  the trainable parameters of all hidden layers, including the   parameters of batch normalization.

\section{Algorithm}
We now describe the two-phase method as a modification of any given first-order base algorithm.   The modified algorithm is proven to have global convergence guarantee under verifiable conditions as shown in the next two sections.  
The base algorithm can be any given first-order algorithm, including both batch and stochastic algorithms, such as gradient descent and stochastic gradient descent with momentum and adaptive step-size. 
 
The  description of the algorithm  is presented in Algorithm \ref{al:train}, where   $\eta_t \odot g^{t}$  represent the Hadamard product of $\eta_t$ and $g^{t}$. 
 Here,  $g^{t}$ represents the rules of the parameter update that correspond to different base algorithms. For example, if we use the (mini-batch) stochastic gradient descent as the base algorithm, $g^{t}$  represents the (mini-batch) stochastic gradient of the loss function with respect to $w$ at the time $t$. 

The first phase of the training algorithm is the same as the base algorithm.
 Then a random Gaussian perturbation is added on all but last layer weights. 
 After the random perturbation, the second training phase starts. In the second training phase, the base algorithm is modified to preserve the rank of the NTK at time $\tau$ after random perturbation, as:
 \[\rank \left(\mathbf{K}(w^k)\right)\ge\rank \left(\mathbf{K}(w^\tau)\right), \quad \forall k = \tau, \tau+1,\dots, t\] 
where the NTK matrix, $\mathbf{K}(w) \in \RR^{nm_y\times nm_y}$, is defined by
$$
 \mathbf{K}(w)=\frac{\partial \vect(f_{X}(w)\T)}{\partial w} \left(\frac{\partial \vect(f_{X}(w)\T)}{\partial w}\right)\T.
$$ 
As two examples, the rank can be preserved by  lazy training on all layer weights or by only training the parameters in the last layer  in the second phase.
 In the next section, we will  develop the global convergence theory for Algorithm  \ref{al:train}.
 
\begin{algorithm}[t!]
\caption{Two-phase modification $\Acal$ of a base  algorithm with global convergence guarantees} \label{al:train} 
$ \ $
\begin{algorithmic}[1] 
\STATE {\bf Inputs:} an initial parameter vector  $w^{0}$,   a time $\tau$ and a base algorithm with updates sequence $(g_t)_t$ and learning rate sequence $(\eta_t )_{t}$.
\STATE
\textbf{$\rhd$ First training phase} 
\FOR {$t = 0,1,\ldots, \tau-1$} 
  \STATE 
  Update parameters: $w^{t+1} =w^{t} - \eta_t \odot g^{t}$
\ENDFOR
\STATE \textbf{$\rhd$ Random perturbation\\ } 
Add noise at time $\tau$,
$$w^{\tau}_{(1:H)}\leftarrow w^{\tau}_{(1:H)}+\delta,$$ where the noise vector $\delta = (\delta_1. \ldots, \delta_H)\in \RR^{d_{1:H}}$ is sampled from a non-degenerate Gaussian distribution: $\delta_h \sim N(0, \sigma_h^2 I_{d_{h}})$ for $h=1,\ldots, H$ .  
\STATE \textbf{$\rhd$ Second training phase} 
\FOR {$t = \tau,\tau+1,\ldots $} 
 \STATE 
  Update parameters: $w^{t+1} =w^{t} - \eta_t \odot g^{t}$,
  where the learning rate $(\eta _{t})_{t>\tau}$  is modified to satisfy the \textit{rank preserving}  condition: for $k = \tau, \tau+1,\dots, t$,
  \[\rank \left(\mathbf{K}(w^k)\right)\ge\rank \left(\mathbf{K}(w^\tau)\right).\]
\ENDFOR   
\end{algorithmic}
\end{algorithm}

\section{Theoretical analysis}
In this section, we prove that the parameter $w^{t}$ in Algorithm \ref{al:train} converges to a global minimum $w^*$ \textit{of all layers} 
under the expressivity condition. As a concrete example, we prove that fully-connected neural networks with softplus nonlinear activations and moderately wide  last hidden layer  satisfy the expressivity condition for all distinguishable training datasets. All proofs in this paper are deferred to Appendix. 

\subsection{Expressivity condition}
Making the right assumption is often the most critical step in theoretical analysis of deep learning. 
The assumption needs to be both weak enough to be useful in practice and strong enough for proving desired conclusions. 
It is often challenging to find the assumption with the right theory-practice trade-off, as typical assumptions that lead to desired conclusions are not weak enough to hold in practice, which contributes to the gap between theory and practice. 
We aim to find the right trade-off by proposing a data-architecture-dependent, time-independent, and  verifiable condition called the \textit{expressivity condition} as a cornerstone for global convergence results. The expressivity condition guarantees the existence of parameters that can interpolate all the training data. 
\begin{assumption}
\label{a:existance_of_w}
\emph{(Expressivity condition)} 
 There exists $w_{(1:H)}$ such that $\varphi(w_{(1:H)})\neq 0$, where 
$\varphi(w_{(1:H)}) := \det(\allowbreak[h_{X}^{(H)}(w_{(1:H)}), \mathbf{1}_{n}][h_{X}^{(H)}(w_{(1:H)}), \mathbf{1}_{n}]^{\top})$.
\end{assumption}
In the expressivity condition, the map $h_{X}^{(H)}$ depends on both architecture and dataset. Such dependency is essential for the theory-practice trade-off; i.e., we obtain a desired conclusion yet only for a certain class of paris of dataset and architecture.  We verify the expressivity condition in our experiments. The expressivity condition is also verifiable as demonstrated below.

\subsection{Real analyticity}
To prove the global convergence, we also require the function $ h^{(H)}_X $ to be real analytic. Since a composition of real analytic functions is real analytic, we only need to check whether each operation satisfies the real analyticity. The  convolution, affine map, average pooling, shortcut skip connection  are all real analytic functions. Therefore, the composition of these layers  preserve real analyticity.

We now prove that the  batch normalization function is also real analytic. The batch normalization that is applied to an output $z$ of an arbitrary coordinate can be written by 
$$
\BN_{\gamma,\beta}(z) = \gamma \frac{z-\mu}{\sqrt{\sigma^2 + \epsilon}} + \beta. 
$$
Here, $\mu$ and $\sigma^2$ depend also on other samples  as
$\mu = \frac{1}{|S|} \sum_{i\in S}  z_i$ and $\sigma^2 = \frac{1}{|S|} \sum_{i\in S} (z_i - \mu)^2$, where $S$ is an arbitrary subset of $\{1,2,\dots,n\}$ such that $z \in \{z_{i} :  i \in S\}$. Then, the following statement holds:   
\begin{proposition} \label{prop:analytic_bn}
Batch normalization function $(z,\beta,\gamma) \mapsto \BN_{\gamma,\beta}(z)$ is real analytic. 
\end{proposition}
We also require the activation function to be analytic. For example, sigmoid, hyperbolic tangents  and softplus activations $\sigma(z)=\ln(1+\exp(\varsigma z))/\varsigma$ are all real analytic functions, with any hyperparameter $\varsigma>0$.  The softplus activation can approximate the ReLU activation  for any desired accuracy as 
$$
\sigma(x) \rightarrow \mathrm{relu}(x) \text{ as } \varsigma\rightarrow \infty.
$$
Therefore, the function $h^{(H)}_X$ is real analytic for  a large class of neural networks (with batch normalization) such as the standard deep  residual networks \citep{he2016identity} with real analytic approximation of ReLU activation via  softplus.

\subsection{Global convergence}
In the following, we assume that the loss function satisfies assumption \ref{a:loss}. 
\begin{assumption}\label{a:loss}
\emph{(Use of common loss criteria)} For any $i \in \{1,\dots,n\}$, the function $\ell_i:q\mapsto\ell(q,y_{i}) \in \RR_{\ge 0}$ is differentiable and convex, and $\nabla \ell_{i}$ is $L_{\ell}$-Lipschitz: i.e.,  $\|\nabla \ell_{i}(q) - \nabla \ell_{i}(q')\| \le L_\ell \|q - q'\|$ for all $q,q' \in \RR$. 
\end{assumption}
Assumption \ref{a:loss} is satisfied by standard loss functions such as the squared loss  $\ell(q,y)= \|q-y\|_2^2$ and
cross entropy loss  $\ell(q,y)=-\sum_{k=1}^{d_y}y_k \log \frac{\exp(q_k)}{\sum_{k'}\exp(q_{k'})}$.
Although the objective function $\Lcal:w \mapsto \Lcal(w)$ used to train a neural network is non-convex in $w$, the loss criterion $\ell_{i}:q  \mapsto\ell(q, y_{i})$ is often convex in $q$. 

Before we state the main theorem, we define the following notation. Let $w^{*} \in \RR^d$ be a global minimum of all layers; i.e., $w^{*}$ is a global minimum of $\Lcal$.  Define $\nu=[\mathbf{0}_{d_{1:H}}\T, \mathbf{1}_{d_{H+1}}\T]\T$ where $\mathbf{0}_{d_{1:H}} \in \RR^{d_{1:H}}$ is the column vector of all entries being zeros and $\mathbf{1}_{d_{H+1}}\in \RR^{d_{H+1}}$ is the column vector of all entries being ones. 
 Let $R^{2}= \min_{\bar{w}_{(H+1)}^* \in Q }\EE[\|\bar w^{*}_{(H+1)}- w^\tau_{(H+1)}\|^{2}]$ where $Q=\argmin_{w_{(H+1)}} \Lcal([(w_{(1:H)}^{\tau})\T , (w_{(H+1)})\T]\T)$.  Now we are ready to state one of our main theorems.
\begin{theorem} \label{thm:general}
  Suppose $H\ge2$, assumptions \ref{a:existance_of_w}  and  \ref{a:loss} hold. Assume that the function $ h^{(H)}_X $ is real analytic. Then, with probability one over a randomly sampled $\delta$, the following two statements hold:
\begin{enumerate}[leftmargin=0.5cm]
\item[(i)] (Gradient descent) if $g^{t}=\nabla_{w^{t}_{(H+1)}} \Lcal(w^{t})$ and $\eta_t = \frac{1}{L_{H}} \nu$ for $t\ge \tau$ with $L_{H}=\frac{L_{\ell}}{n}  \sum_{i=1}^n \|[h^{(H)}(x_{i},w_{(1:H)}^{\tau}), 1]\|_2^2$,
then  for any $t> \tau$,
$$
\Lcal(w^{t}) - \Lcal(w^{*} )\le  \frac{R^{2}L_{H} }{2(t-\tau)}. 
$$

\item[(ii)] 
(SGD) if $\EE [g^{t}|w^{t}]=\nabla_{w^{t}_{(H+1)}} \Lcal(w^{t})$ (almost surely), $\EE[\|g^{t}\|^{2}]\le G^{2} $ and  $\eta_t= \bar \eta_t\nu$ for $t\ge \tau$ with $\bar \eta_t \in \RR$ satisfying that $\bar\eta_t \ge 0$, $\sum_{t=\tau}^{\infty} \bar\eta_t^2 < \infty $ and $\sum_{t=\tau}^{\infty} \bar\eta_t = \infty$, then  for any $t> \tau$,
\begin{align*}
\EE[\Lcal(w^{t^*}) ]  -\Lcal(w^{*} )\le    \frac{R^{2}+G^{2}  \sum_{k=\tau}^t\bar \eta_{k}^{2}}{ 2 \sum_{k=\tau}^t \bar \eta_{k}},
\end{align*}
where $t^*\in \argmin_{k \in \{\tau,\tau+1,\dots, t\}}\Lcal_{}(w_{}^{k})$.  
\end{enumerate}
 
\end{theorem}
In particular, Theorem \ref{thm:general} part (ii) shows that  if we choose $\bar \eta_t \sim O(1/\sqrt{t})$, we have $\lim_{t\rightarrow \infty} \frac{\sum_{k=\tau}^t \bar \eta_k^2}{\sum_{k=\tau}^t \bar \eta_k}=0$   and the optimality gap becomes
$$
\EE[\Lcal(w^{t^*}) ] -\Lcal(w^{*} )= \tilde O(1/\sqrt{t}).
$$ 

\subsection{Example}
As a concrete example that satisfies all the conditions in theorem \ref{thm:general}, we  consider full-connected deep networks using softplus activation with a wide last hidden layer. In the case of fully-connected networks, the output of the last hidden layer can be simplified to 
\begin{align} \label{eq:fully_connected_1}
h^{(H)}_X(w_{(1:H)})_{ij}=h^{(H)}(x_{i},w_{(1:H)}))_{j} \in \RR,
\end{align}
where $h^{(l)}(x_{i},w_{(1:l)}) \in \RR^{1 \times m_l}$ is defined by
\begin{align} \label{eq:fully_connected_2}
h^{(l)}(x_{i},w_{(1:l)})=\sigma(h^{(l-1)}(x_{i},w_{(1:l-1)})W^{(l)}+b^{(l)})
\end{align}
for $l=1, 2, \dots, H$ with  $h^{(0)}(x_{i},w_{(1:0)}):=x_i\T \in \RR^{1 \times m_x}$. Here,   $W^{(l)}\in \RR^{ m_{l-1} \times m_{l}}$ and $b^{(l)}\in \RR^{1 \times m_{l}}$ are the weight matrix and the bias term  of the $l$-th  hidden layer. Also, $m_l$ represents the number of neurons at the $l$-th hidden layer. 
Since $h^{(H)}_X$ is the composition of affine functions and real analytic activation functions (i.e.,  softplus activation $\sigma$), the function $ h^{(H)}_X $ is real analytic.

In theorem \ref{thm:fcnet}, we show that the expressivity condition is also satisfied for fully-connected networks, for training datasets that satisfy the following input distinguishability assumption. 
\begin{assumption}\label{a:input}
\emph{(Input distinguishability)}
$\|x_i\|^2- x_i\T x_j > 0$ for any $x_i,x_j \in S_{x}$ with $i\neq j$.
\end{assumption}

\begin{theorem} \label{thm:fcnet}
Suppose assumption \ref{a:input} hold. Assume that $h_X^{(H)}$ is defined by equations \eqref{eq:fully_connected_1}-\eqref{eq:fully_connected_2}  with softplus activation $\sigma$ and $H\ge2$ such that $\min(m_1,\dots,m_{H-1})\ge \min(m_x, n)$, and $m_H \ge n$. Then, assumption \ref{a:existance_of_w} hold true.
\end{theorem}

In Theorem  \ref{thm:fcnet}, the case of $\min(m_1,\dots,m_{H-1})=m_x$ is allowed.
That is, all of the $1,2,\dots,H-1$-th hidden layers are allowed to be narrow (instead of wide) with the number of neurons to be   $m_{x}$, which is typically smaller than $n$. A previous paper recently proved that   gradient descent  finds a global minimum in a lazy training regime (i.e., the regime where NTK approximately remains unchanged  during training) with $d = \tilde \Omega( m_{y}n + m_x H^2+H^{5})$ \cite{kawaguchi2019gradient}. In contrast, Theorem  \ref{thm:fcnet} only requires $d \ge (m_y+m_{x})n+m_x^2H$ and   allows NTK to change significantly during training.

Assumption \ref{a:input} used in Theorem  \ref{thm:fcnet} can be easily satisfied, for example, by normalizing the input features for $x_{1},\dots,x_n$ so that $\|x_{i}\|^{2}=\|x_{j}\|^2$. With the normalization,  the condition is satisfied as long as $\|x_{i}-x_{j}\|^2>0$ for $i\neq j$   since  $\frac{1}{2}\|x_{i}-x_{j}\|^2=\|x_{i}\|^{2}-x_{i}^{\top} x_{j}$. 
In general, normalization is not necessary, for example, orthogonality on   $x_{i}$ and $x_{j}$ along with $x_{i} \neq 0$ satisfies the condition.

\subsection{Global convergence with lazy training in the second phase }
In the previous section, we did not assume the rank preservation condition. Instead, we considered the special  learning rate  $\eta_t$ in the second phase to  keep submatrices of the kernel matrix  $\mathbf{K}(w)$ unchanged during the second phase ($t \ge \tau$).
In this section, we show that algorithm \ref{al:train} can still ensure the global convergence with a standard uniform learning rate $\eta _{t}= \frac{2 \bar \eta}{L}\mathbf{1}_d$, as long as  the rank preservation condition is satisfied.

\begin{theorem} \label{thm:all}
Let $\eta _{t}= \frac{2 \bar \eta}{L}\mathbf{1}_d$ with $\bar \eta \in \RR$ for $t \ge \tau$. Suppose that $H\ge2$ and  the following three assumptions  hold:
\begin{itemize}[leftmargin=0.7cm]
\item
    Assumption \ref{a:existance_of_w} (expressivity condition)
    \item Assumption \ref{a:loss}, along with $\|\nabla\Lcal(w)-\nabla\Lcal(w')\|\le L \|w-w'\|$ for all $w,w'$ in the domain of $\Lcal$.
    \item 
    (rank preserving condition) $\rank \left(\mathbf{K}(w^k)\right)\ge\rank \left(\mathbf{K}(w^\tau)\right) $  for for all $k \in \{\tau+1, \tau+2,\dots, t\}$.
\end{itemize}
Then, the following  statement hold for any $t>\tau$: 
$$
\Lcal(w^{t^*}) - \Lcal(w^{*} )\le   \frac{1}{\sqrt{(t-\tau)+1}}    \sqrt{\frac{L \bar R^{2}(\Lcal(w^{\tau})- \Lcal(w^{*} ) )}{2 \bar \eta(1- \bar \eta)}}.
$$ 
where $t^*\in \argmin_{k \in \{\tau,\tau+1,\dots, t\}}\Lcal(w^{k})$ and 
$
\bar R = \max_{\tau\le k \le t} \min_{\hat \omega^{k}\in \bar Q_{k} }\|(\nu \odot w^{k} )- \hat \omega^{k}\|
$ 
with 
$
\bar Q_k = \argmin_{\hat \omega \in \RR^d} \frac{1}{n} \sum_{i=1}^n \ell\left(\sum_{j=1}^{d} \hat \omega_j \frac{\partial f(x_i,w^{k})\T}{\partial w_j},y_{i} \right).
$
\end{theorem}

 Theorem \ref{thm:all} shows that  using the lazy training that preserves the rank of the NTK matrix during the second phase can  ensure the  global convergence  for Algorithm \ref{al:train}. Therefore, the proposed two-phase training algorithm provides a new perspective for the lazy training regime. That is, NTK in Algorithm \ref{al:train}  is allowed to change significantly during the first phase training $t<\tau$ to learn the features or representation beyond the random features induced by the data-independent NTK at initialization. Our two-phase method allows the lazy training with the learned data-dependent NTK at time $\tau$, which is often a better representation of practical dataset than the NTK at random initialization.

The property of the lazy training now depends on the quantities at time $\tau$ instead of time $t=0$. For example, if we conduct the lazy training with over-parameterization, then the number of neurons required per layer depends on the residual error and the minimum eigenvalue of NTK\ at time $\tau$, instead of time $t=0$. This could potential lead to global convergence theory with weaker assumptions. Thus, the two-phase algorithm opens up a new direction of future research for applying the lazy training theory to the data-dependent NTK obtained at the end of first phase training. We can define the domain of $\Lcal$ to be a sublevel set around an initial point to satisfy the Lipschitz continuity.

\section{Proof idea and key challenges}
For global convergence for optimization of deep neural networks, recent results rely on different assumptions, such as over-parameterization and  initial conditions on gradient dynamics. Those different assumptions are essentially used in proofs to enforce the full rankness of the NTK matrix and the corresponding feature matrix during training. If the feature matrix is of full rank, then the global convergence is ensured (and the convergence rate depends on the minimum eigenvalue of the NTK matrix).

In order to apply our theory for practical settings, we want \textit{data-dependency} in two key aspects. First, we want the feature matrix to be data-dependent and to change significantly during training, in order to learn data-dependent features. In contrast, the various  assumptions in previous works essentially make the feature matrix to be approximately data-independent. Second, we want the global convergence results to hold data-dependently for a certain class of practical datasets. Instead, previous global convergence  results need to hold data-independently, or for linearly-separable datasets or  synthetic datasets generated by simple models (e.g.,  Gaussian mixtures). Because of these differences, there needs to be new proof strategies instead of adopting previous assumptions and their proof ideas.

\subsection{Proof for general networks}  
The first step in our proof is to show that the feature matrix is of full rank with probability one over random entries of the parameter vector. The global convergence  then follows from the full rankness (as shown in the complete proof in Appendix).

A challenge in the proof for the general case is to make the right assumption as discussed above. If we assume significant over-parameterization, then proving the full rankness is relatively easy, but it limits the applicability. Indeed, we want to allow deep networks to have narrow layers. 

Although the entries of the parameter vector after random perturbation have independent components, the entries of the feature matrix are dependent on each other. Indeed, the entries of the feature matrix are the outputs of nonlinear and non-convex functions of the entries of the parameter vector. Therefore, we cannot use elementary facts from linear algebra and random matrix theory with i.i.d. entries to prove the full rankness of the feature matrix.

Instead, we take advantage of the fact on the zero set of a real analytic function: if a function is real analytic and not identically zero, then the Lebesgue measure of its zero set is zero \citep{mityagin2015zero}. To utilize this fact, we define a function $\varphi(w_{(1:H)})=\det(h_{X}^{(H)}(w_{(1:H)})h_{X}^{(H)}(w_{(1:H)})^{\top})$. We then observe that $\varphi$ is  real analytic  since $h_{X}^{(H)}$ is assumed to be real analytic and a composition of real analytic functions is real analytic.
Furthermore, since the rank of $h_{X}^{(H)}(w_{(1:H)})$ and the rank of the Gram matrix are equal, we have that $\{w_{(1:H)}\in \RR^{d_{1:H}}: \rank(h_{X}^{(H)}(w_{(1:H)})) \neq n \}
=\{w_{(1:H)}\in \RR^{d_{1:H}}:\varphi(w_{(1:H)})=0\}.
$
Since $\varphi$ is analytic, if $\varphi$ is not identically zero ($\varphi\neq 0$),  the Lebesgue measure of its zero set 
$
\{w_{(1:H)}\in \RR^{d_{1:H}}:\varphi(w_{(1:H)})=0\}
$
is zero. 
Therefore, if $\varphi(w_{(1:H)})\neq 0$ for some $w_{(1:H)}\in \RR^{d_{1:H}}$, the  Lebesgue measure of the set 
$$
\{w_{(1:H)}\in \RR^{d_{1:H}}: \text{$\rank(h_{X}^{(H)}(w_{(1:H)})) \neq n$} \}
$$
is zero. Then, from the full rankness of the feature (sub)matrix $h_{X}^{(H)}(w_{(1:H)})$, we can ensure the global convergence, as in the previous papers with over-parameterization and neural tangent kernel.

Based on the above proof idea, as long as there exists  a $w_{(1:H)}\in \RR^{d_{1:H}}$ such that $\varphi(w_{(1:H)})\neq 0$, then we can conclude the desired global convergence. The key observation is that this condition is a time-independent and easily verifiable  condition in practice. This condition is defined as assumption \ref{a:existance_of_w}. We verify that assumption \ref{a:existance_of_w} holds in experiments numerically for deep ResNets and in theory for fully-connected networks. We complete the proof in Appendix.

\subsection{Proof for fully-connected networks}  
Without our result on the general case, a challenge to prove the global convergence of fully-connected networks lies in the task of dealing with narrow hidden layers; i.e., in the case of $\min(m_1,\dots,m_{H-1})= m_x < n$. In the case of $\min(m_1,\dots,m_{H-1}) \ge n$, it is easy to see that $k$-th layer with $m_k \ge n$ can preserve  rank $n$ for the corresponding matrices. In the case of $\min(m_1,\dots,m_{H-1})= m_x<n$, however, it cannot preserve rank $n$, and deriving the global convergence is non-trivial. 

With our result for the general case, however, our only remaining task is to show the existence of  a $w_{(1:H)}\in \RR^{d_{1:H}}$ such that $\varphi(w_{(1:H)})\neq 0$. Accordingly, we complete the proof in appendix by constructing such a $w_{(1:H)}\in \RR^{d_{1:H}}$ for the fully-connected networks with narrow layers.

\begin{figure}[b!]
\centering
\begin{subfigure}[b]{0.49\columnwidth}
  \includegraphics[width=1.0\columnwidth,height=0.6\columnwidth]{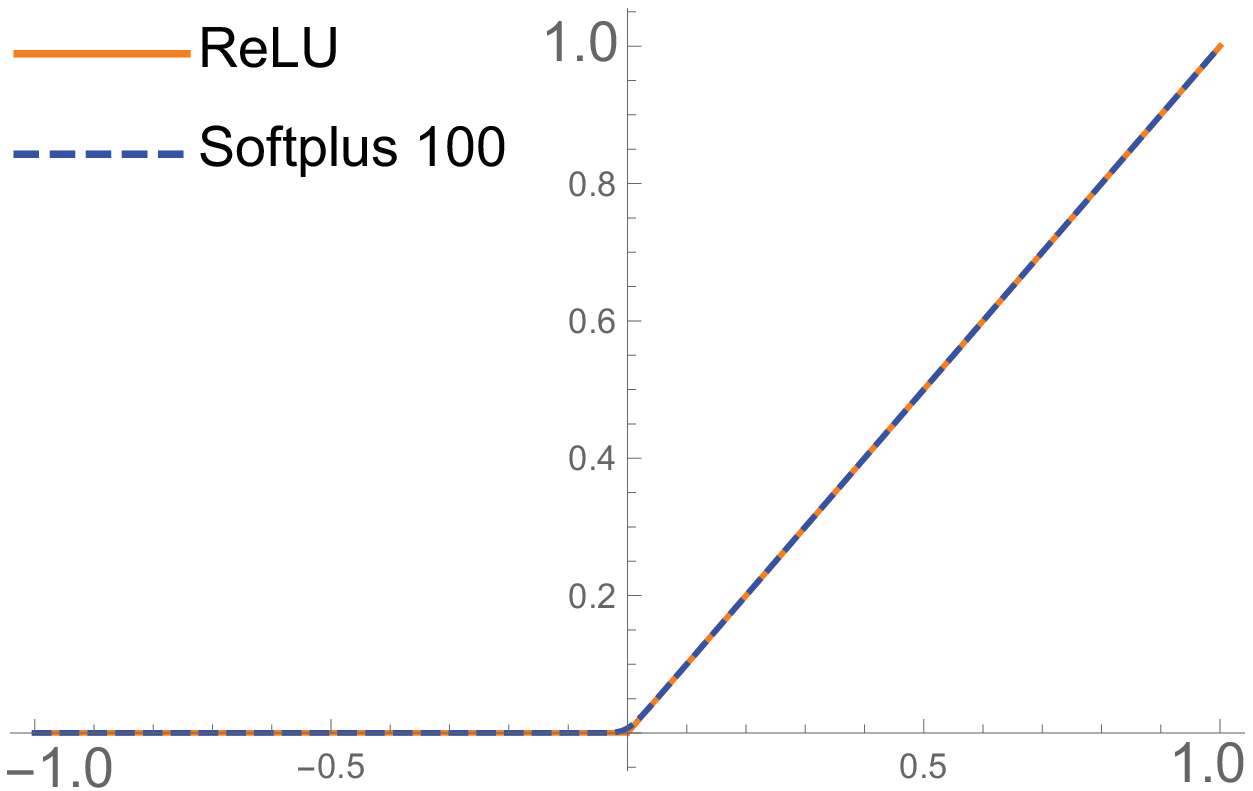}
\end{subfigure}
\begin{subfigure}[b]{0.49\columnwidth}
  \includegraphics[width=1.0\columnwidth,height=0.6\columnwidth]{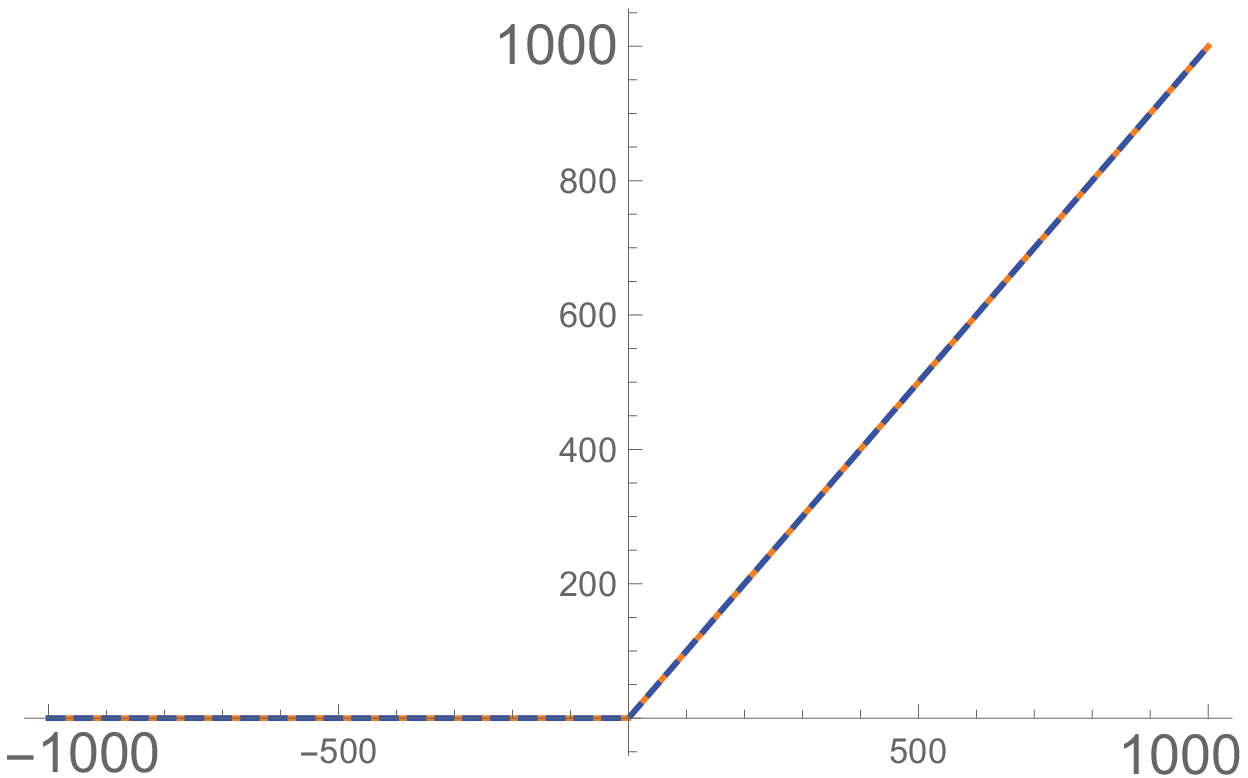}
\end{subfigure}
\captionof{figure}{ReLU   versus Softplus  with $\varsigma=100$.  
} 
\label{fig:softplus_vs_relu}
\end{figure}

\begin{table*}[t!]
\centering \renewcommand{\arraystretch}{0.1} 
\caption{Test errors (\%) of base and $\Acal(\text{base})$ with guarantee where the operator $\Acal$ maps any given first-order training algorithm to the two-phase version of the given algorithm with theoretical guarantees.  The numbers indicate the mean test errors (and standard deviations in parentheses) over five random trials. The  column of `Augmentation' shows `No' for  no  data augmentation, and `Yes' for data augmentation. The expressivity condition (assumption \ref{a:existance_of_w}) was numerically verified to all datasets.} \label{tbl:test_error} 
\begin{tabular}{lccccc}
\toprule
Dataset & \# of training data & Expressivity Condition& Augmentation &  Base & {\fontsize{8.5}{8.5}\selectfont{$\Acal(\text{base})$ with guarantee}}   \\
\midrule
MNIST & 60000 & Verified & No  & 0.41 (0.02)  & 0.38 (0.04) \vspace{4pt}\\

 &  & & Yes  &  0.34 (0.03) & 0.28 (0.04)  \\
\midrule
CIFAR-10 & 50000 & Verified & No   & 13.99 (0.17) & 13.57 (0.32)\vspace{4pt}\\

 &    &  & Yes  & 7.01 (0.18) & 6.84 (0.16)\\
\midrule
CIFAR-100 & 50000 & Verified &   No   & 41.43 (0.43) & 40.78 (0.22) \vspace{4pt}\\

 & & & Yes & 27.92 (0.36) & 27.41 (0.58) \\
\midrule
SVHN & 73257 & Verified &  No & 4.51 (0.04) & 4.50 (0.09) \vspace{4pt}\\
 &  &  & Yes  & 4.32 (0.06) & 4.16 (0.16) \\
\bottomrule
\end{tabular} 
\end{table*}

\section{Experiments}
In this section, we study the empirical aspect of our method. The network model we work with is the standard (convolutional) pre-activation ResNet with $18$ layers \citep{he2016identity}. To satisfy all the assumptions in our theory, we added a fully-connected last hidden layer  of  $cn$ neurons with a small constat $c = 1.1$ and set the nonlinear activation functions of all layers to be softplus $\sigma(z)=\ln(1+\exp(\varsigma z))/\varsigma$ with $\varsigma=100$, as a real analytic function that approximates the ReLU activation. This approximation is of high accuracy  as shown in Figure \ref{fig:softplus_vs_relu}. As discussed above, Proposition \ref{prop:analytic_bn} implies that the function $h^{(H)}_X$ for the ResNet  is real analytic. The loss function we work with is cross-entropy loss, which satisfies assumption \ref{a:loss}. Therefore, the only assumption in theorem \ref{thm:general} that is left to verify is assumption \ref{a:existance_of_w}. In the following subsection, we numerically verify assumption \ref{a:existance_of_w}.

\subsection{Verification of expressivity condition}
Assumption \ref{a:existance_of_w} assumes that the network satisfy expressivity condition, which only requires an existence of a  $w_{(1:H)}$ such that $\varphi(w_{(1:H)})\neq 0$. Here, $\varphi(w_{(1:H)})\neq 0$ is  implied by $\rank([h_{X}^{(H)}(w_{(1:H)}), \mathbf{1}_{n}])= n$. In other words, if we find one $w_{(1:H)}$ with $\rank([h_{X}^{(H)}(w_{(1:H)}), \mathbf{1}_{n}])= n$, then assumption \ref{a:existance_of_w}
is ensured to hold true. 
A simple way to find such a  $w_{(1:H)}$ is to randomly sample a single $w_{(1:H)}$ and check the condition of $\rank([h_{X}^{(H)}(w_{(1:H)}), \mathbf{1}_{n}])= n$.   

Table \ref{tbl:test_error} column $3$  summarizes the results of the verification of assumption \ref{a:existance_of_w} for various datasets. Here, we used a randomly sampled $w_{(1:H)}$ returned from  the default initialization of the ResNet with version $1.4.0.$ of PyTorch \cite{NEURIPS2019_9015}  by setting  random seed to be $1$. This initialization is based on the implementation of \cite{he2015delving}.
The condition of $\rank([h_{X}^{(H)}(w_{(1:H)}), \mathbf{1}_{n}])= n$ was checked by using   \textit{numpy.linalg.matrix\_rank} in NumPy version $1.18.1$ with  the default option (i.e., without any arguments except the matrix $[h_{X}^{(H)}(w_{(1:H)}), \mathbf{1}_{n}]$), which uses the standard method from \citep{press2007numerical}.

\begin{table}[t!]
\fontsize{8pt}{8pt}\selectfont
\centering 
\renewcommand{\arraystretch}{1.5} 
\caption{Test errors (\%) of   $\Acal(\text{base})$  with guarantee for  Kuzushiji-MNIST with different hyperparameters  $\tau = \tau_0 T$ and $\delta= \delta_0 \epsilon$. The numbers indicate the mean test errors (and standard deviations in parentheses) over three random trials. The expressivity condition (Assumption \ref{a:existance_of_w}) was numerically verified to hold for Kuzushiji-MNIST as well. } \label{tbl:hyper_param_search} 
\begin{tabular}{|l||*{4}{c|}}
\hline
\backslashbox{$\delta_0$}{$\tau_0$}
&\makebox[3em]{0.4} 
&\makebox[3em]{0.5}
&\makebox[3em]{0.6}
&\makebox[3em]{0.8}
\\\hline\hline
0.0001 & 2.10 (0.14) 
& 
2.06 (0.07)
& 
2.02 (0.09) & 2.01 (0.05)
\\\hline
0.001 & 2.06 (0.07) 
&
2.11 (0.12)
& \uline{1.92} (0.06) & 2.01 (0.12) \\\hline
0.01  & 2.25 (0.11)
& 2.25 (0.17)
& 2.25 (0.11) & 2.09 (0.08) \\\hline
\end{tabular}
\end{table}

\subsection{Performance}
In the following, we compare the generalization performances of the  two-phase training algorithm over different hyper-parameters' choices and with the baseline algorithm.

\subsubsection{Experimental setting}
We fixed all hyper-parameters of the base algorithm a priori across all different datasets by using a standard hyper-parameter setting of  SGD (following the setting of \citealp{kawaguchi2020ordered}), instead of aiming for  state-of-the-art test errors with a possible issue of over-fitting to  test and validation datasets \citep{dwork2015reusable,rao2008dangers}. Concretely, we fixed the mini-batch size to be 64, the weight decay rate to be $10^{-5}$, the momentum coefficient to be $0.9$, the first phase learning rate to be   $\eta_{t}=0.01$ and  the second phase learning rate to be $\eta_{t}=0.01 \times [\mathbf{0}_{d_{1:H}}\T, \mathbf{1}_{d_{H+1}}\T]\T$ to only train the last layer. The last epoch $T$ was fixed a priori as $T=100$ without data augmentation and $T=400$ with data augmentation.

\subsubsection{Choice of $\tau$ and $\delta$}
Now we discuss the choice of hyper-parameters for the time of transition $\tau$ and for the size of the noise $\delta$. Instead of potentially overfitting hyper-parameters to each dataset, we used a different dataset, Kuzushiji-MNIST \cite{clanuwat2019deep}, to fix all the hyper-parameters of Algorithm \ref{al:train} across all other datasets. That is,  we used Kuzushiji-MNIST with  different hyperparameters' values $\tau_0=0.4, 0.5, 0.6, 0.8$ and $\delta_0=0.0001, 0.001, 0.01$, where   $\tau = \tau_0 T$ and $\delta= \delta_0 \epsilon$. Here, $\epsilon\sim \Ncal(\mathbf 0,I_{d})$ where $I_{d}$ is the $d \times d$ identity matrix. Based on  the results from Kuzushiji-MNIST in Table \ref{tbl:hyper_param_search}, we fixed $\tau_0=0.6$ and $\delta_0=0.001$ for all datasets. 

\subsubsection{Generalization and optimization  comparison} We now compare the performances between the  base algorithm and its two-phase modified version. For generalization aspect, as shown in the last three columns of table~\ref{tbl:test_error}, the modified algorithm improved the test errors consistently over the four datasets with and without data augmentation. This suggests that the modified version of the base algorithm is competitive with the base algorithm for generalization performance. 

For optimization aspect, figure \ref{fig:train_loss} shows that the two-phase training algorithm indeed improves training loss values of the base algorithm in the second phase without changing any hyper-parameters (e.g., learning rate and momentum) of the base algorithm, as expected from our theory.  These results suggest that the two-phase algorithm can provide global convergence guarantees to a given base algorithm without hurting the generalization performance.

\begin{figure}[t!]
\centering
\begin{subfigure}[b]{0.495\columnwidth}\centering
  \includegraphics[width=\columnwidth,height=0.7\columnwidth]{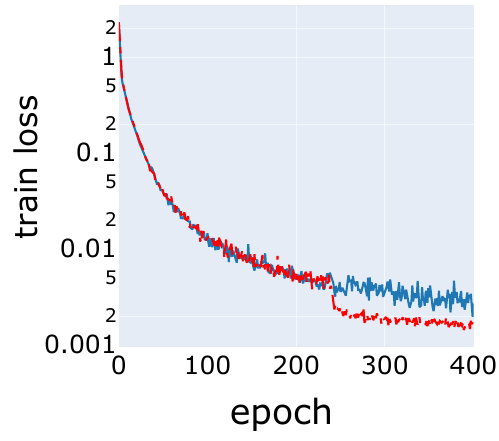}
  \caption{CIFAR-10} 
\end{subfigure} 
\begin{subfigure}[b]{0.495\columnwidth}\centering
  \includegraphics[width=\columnwidth,height=0.7\columnwidth]{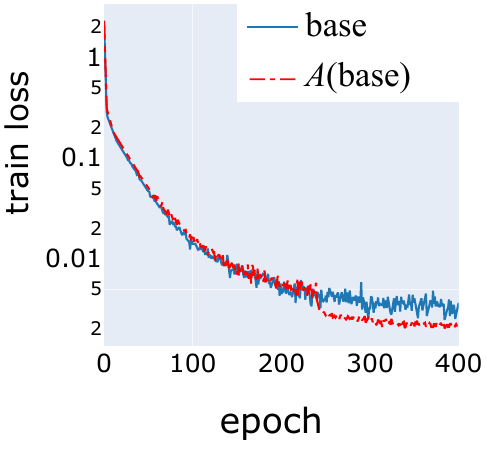}
  \caption{SVHN} 
\end{subfigure}  
\caption{Training loss of base and $\Acal(\text{base})$ with guarantee: the plots show the mean values over five random trials.}
\label{fig:train_loss} 
\end{figure}

\section{Conclusion}
In this paper, we proposed a two-phase method that modifies any given first-order optimization  algorithm to have global convergence guarantees without degrading practical performances of the given algorithm. The  conditions for global convergence are mathematically proven to hold for fully-connected deep networks with wide last hidden layer (while all other layers are allowed to be narrow). The conditions are also numerically verified for deep ResNet with batch normalization under various standard classification datasets. 

The two-phase training method opens up  a new future research direction to study the  use of the novel NTK regime with  \textit{learned representation} from data unlike the standard NTK regime near random initialization. Extending our theoretical analysis on a larger class of NTK with learned representation for global convergence and for generalization performance would be an interesting future direction.   

As the global optimal parameters can often achieve near zero training loss, we could expect models trained by our modified algorithm to have the benefits from  \textit{terminal phase of training} \cite{papyan2020prevalence} potentially for better generalization performance, robustness, and interpretability. Verifying these benefits would be sensible directions for future work. An extension of our theory and algorithm to implicit deep learning \cite{bai2019deep,el2019implicit,kawaguchi2021on} would be another interesting future direction.

\section*{Acknowledgments}
This work is partially supported by the Center of Mathematical Sciences and Applications at Harvard University. The authors thank Yiqiao Zhong, Mengyuan Yan for discussion.

\bibliography{all}

\clearpage
\onecolumn
\appendix

\begin{center}
\textbf{\LARGE 
 Appendix
\vspace{5pt}}
\end{center}

For the pedagogical purpose, we first present a  proof of theorem \ref{thm:general} for fully-connected neural networks, which illustrates our key ideas in the proof of theorem \ref{thm:general} in a more concrete and  familiar manner. 
Then we prove theorem \ref{thm:general} for general networks under expressivity condition (assumption \ref{a:existance_of_w}), and then we come back to fully-connected neural networks and  prove theorem \ref{thm:fcnet}. After we proved the first two theorems for global convergence, we prove proposition \ref{prop:analytic_bn} to show that batch-normalization layer is real analytic. Then, building upon these proofs, we prove theorem \ref{thm:all}. 

For the reproducibility of the experimental results, we provided the code used in the experiments with instructions as another supplementary material.

\section{Proof of Theorem \ref{thm:general} for fully-connected neural networks}
We start by defining the matrix notations:   
$$
f_{X}(w)= \begin{bmatrix}
f(x_{1},w) \\
\vdots \\
f(x_{n},w) \\
\end{bmatrix}\in \RR^{n \times m_y}, 
$$
$$
h^{(l)}_X(w_{(1:l)}) = \begin{bmatrix}
h^{(l)}(x_{1},w_{(1:l)}) \\
\vdots \\
h^{(l)}(x_{n},w_{(1:l)}) \\
\end{bmatrix} \in \RR^{n \times m_l} , 
$$
  
$$ 
f_X(w) = [h_{X}^{(H)}(w), \mathbf{1}_{n}] \begin{bmatrix} W^{(H+1)} \\  b^{(H+1)} \end{bmatrix}  \in \RR^{n \times m_y},
$$
$$
h^{(l)}_X(w_{(1:l)}) = \sigma\left( [h_{X}^{(l-1)}(w_{(1:l-1)}), \mathbf{1}_{n}] \begin{bmatrix} W^{(l)} \\  b^{(l)} \end{bmatrix}  \right)\in \RR^{n \times m_l} ,
$$
where   $\mathbf{1}_{n}=[1,1,\dots,1]^{\top} \in \RR^{n}$ and $\sigma$ is applied element-wise (by  overloading  of the notation $\sigma$). Let $M_k$ be the $k$-th column of the matrix $M$; e.g., $f_{X}(w)_k\in \RR^{n}$ and $h^{(l)}_X(w_{(1:l)})_{k}\in \RR^{n}$ are the $k$-th column vectors of $f_{X}(w)$ and $h^{(l)}_X(w_{(1:l)})$, respectively, which can be written by

$$ 
f_X(w)_{k} = [h_{X}^{(H)}(w_{(1:H)}), \mathbf{1}_{n}] \begin{bmatrix} W^{(H+1)}_k \\  b^{(H+1)}_k \end{bmatrix}  \in \RR^{n},
$$
$$ 
h^{(l)}_X(w_{(1:l)})  _{k}= \sigma\left( [h_{X}^{(l-1)}(w_{(1:l-1)}), \mathbf{1}_{n}] \begin{bmatrix} W_{k}^{(l)} \\  b_{k}^{(l)} \end{bmatrix}  \right)\in \RR^{n}.
$$
Let $M_{1:n,1:n}\in \RR^{n \times n}$ be the fist $n \times n$ block matrix of the matrix $M$, $v_{1:n} \in \RR^n$ be the vector containing the first $n$ elements  of the vector $v$ in the same order, and $I_{n}$ be the  $n \times n$ identity matrix. Let $M \otimes M'$ be the Kronecker product of $M$ and $M'$.

The following  lemma is used to take care of  the common and challenging case  of $\min(m_1,\dots,m_{H-1})\ge  m_{x}$, whereas we consider the uncommon and easier case of $\min(m_1,\dots,m_{H-1})\ge  n$ later.

\begin{lemma} \label{lemma:rank_step2}
Suppose Assumption \ref{a:input} hold. Assume that  $H\ge2$,   $\min(m_1,\dots,m_{H-1})\ge m_x$, and $m_H \ge n$. Then, the Lebesgue measure of the  set 
$
 \{w_{(1:H)}\in \RR^{d_{1:H}}: \rank(h_{X}^{(H)}(w_{(1:H)})) \neq n \}
$
is zero.
\end{lemma}
\begin{proof} Define $\varphi(w_{(1:H)})=\det(h_{X}^{(H)}(w_{(1:H)})h_{X}^{(H)}(w_{(1:H)})^{\top})$, which is   analytic  since $\sigma$ is analytic.
Furthermore, we have that $\{w_{(1:H)}\in \RR^{d_{1:H}}: \rank(h_{X}^{(H)}(w_{(1:H)})) \neq n \}
=\{w_{(1:H)}\in \RR^{d_{1:H}}:\varphi(w_{(1:H)})=0\},
$
since the rank of $h_{X}^{(H)}(w_{(1:H)})$ and the rank of the Gram matrix are equal.
Since $\varphi$ is analytic, if $\varphi$ is not identically zero ($\varphi\neq 0$),  the Lebesgue measure of its zero set 
$
\{w_{(1:H)}\in \RR^{d_{1:H}}:\varphi(w_{(1:H)})=0\}
$
is zero \citep{mityagin2015zero}. 
Therefore, if $\varphi(w_{(1:H)})\neq 0$ for some $w_{(1:H)}\in \RR^{d_{1:H}}$, the  Lebesgue measure of the set $\{w_{(1:H)}\in \RR^{d_{1:H}}: \text{$\rank(h_{X}^{(H)}(w_{(1:H)})) \neq n$} \}$ is zero. 

Accordingly, we now constructs a $w_{(1:H)}\in \RR^{d_{1:H}}$such that $\varphi(w_{(1:H)})\neq 0$. Set $W^{(1)}_{1:m_x,1:m_x}=\cdots=W^{(H-1)}_{1:m_x,1:m_x}=I_{1:m_x}$,  $b^{(1)}_{1:m_x}=\alpha \mathbf{1}_{m_x}\T$, and all others to be zero. Let $\mathbf{h}(x)=h^{(H-1)}(x,w_{(1:H-1)})_{1,1:m_x}$. Then, 
\begin{align*}
\mathbf{h}(x)=\sigma^{\circ H-1}(x\T +\alpha\mathbf{1}_{m_x}\T), 
\end{align*}
where $\sigma^{\circ l}(z)=\sigma(\sigma^{\circ l-1}(z))$ for $l \ge 1$ with         $\sigma^{\circ0}(z)=z$.
Since\begin{align*}
\varsigma\sigma(z)= \ln(1 + e^{\varsigma z})= \ln((e^{-\varsigma z}+1)e^{\varsigma z})= \ln((e^{-\varsigma z}+1)e^{\varsigma z})=\ln(e^{\varsigma z})+ \ln(e^{-\varsigma z}+1)=\varsigma z+ \ln(1+e^{-\varsigma z}),
\end{align*}
we have 
$$\sigma(z)=z+ \ln(1+e^{-\varsigma z})/\varsigma,
$$ 
and 
\begin{align*}
\sigma^{\circ H-1}(z)&=\sigma^{\circ H-2}(z)+ \ln(1+e^{-\varsigma \sigma^{\circ H-2}(z)})/\varsigma
\\ & =z+ \sum_{l=0}^{H-2} \ln(1+e^{-\varsigma \sigma^{\circ l}(z)})/\varsigma.
\end{align*}
Thus, 
\begin{align*} 
\text{$\mathbf{h}(x)=\sigma^{\circ H-1}(x\T +\alpha\mathbf{1}_{m_x}\T) = x\T +\alpha\mathbf{1}_{m_x}\T + \psi( x\T +\alpha\mathbf{1}_{m_x}\T) $}.
\end{align*}
where $\psi(z)= \sum_{l=0}^{H-2} \ln(1+e^{-\varsigma \sigma^{\circ l}(z)})/\varsigma$. Here, $\psi(z) \rightarrow 0$ as $z\rightarrow \infty$, since 
$
\sigma^{\circ l}(z)=z+ \sum_{k=0}^{l-1} \ln(1+e^{-\varsigma \sigma^{\circ k}(z)})/\varsigma \ge z 
$
and $\ln(1+e^{-\varsigma \sigma^{\circ k}(z)}) \ge 0$.  

From Assumption \ref{a:input}, there exists $c>0$ such that $\|x_i\|_2^2- x_i\T x_j > c$ for any $x_i,x_j \in S_{x}$ with $i\neq j$. Using such a $c>0$ as well as a $\alpha'>0$,  set $W^{(H)}_i=\alpha'x_{i}\in \RR^{m_x}$ and $b^{(H)}_i=-\alpha'\alpha\mathbf{1}_{m_x}\T x_{i}^{}+\alpha'(c/2 - \|x_i \|^2)$ for $i=1,\dots,n$. 
Set all other weights and bias to be zero. Then,
for any $i\in\{1,\dots,n\}$,
\begin{align*}
h^{(H)}_X(w_{(1:H)}) _{ii} &=\sigma(\alpha'(x\T _{i}+\alpha\mathbf{1}_{m_x}\T + \psi( x\T _{i}+\alpha\mathbf{1}_{m_x}\T))x_{i} -\alpha'\alpha\mathbf{1}_{m_x}\T x_{i}^{}+\alpha'(c/2 - \|x_i \|^2)), 
\\ & =\sigma(\alpha'(c/2+\psi( x\T _{i}+\alpha\mathbf{1}_{m_x}\T)x_{i})),
\end{align*}
and for any $j\in\{1,\dots,n\}$ with  $j \neq i$,
\begin{align*}
h^{(H)}_X(w_{(1:H)})_{ij} &=\sigma(\alpha'(x\T _{i}+\alpha\mathbf{1}_{m_x}\T + \psi( x\T _{i}+\alpha\mathbf{1}_{m_x}\T))x_{j} -\alpha'\alpha\mathbf{1}_{m_x}\T x_{j}^{}+\alpha'(c/2 - \|x_j \|^2))
\\ & =\sigma(\alpha'(x\T _{i}x_{j}-\|x_j \|^2+c/2+\psi( x\T _{i}+\alpha\mathbf{1}_{m_x}\T)x_{j})) \\ & \le \sigma(\alpha'(-c/2+\psi( x\T _{i}+\alpha\mathbf{1}_{m_x}\T)x_{j})),   
\end{align*}  
where the last inequality used the monotonicity of $\sigma$ and $\|x_i\|_2^2- x_i\T x_j > c$.

Since $\sigma(\alpha'c/2)\rightarrow \infty$ and $\sigma(-\alpha'c/2) \rightarrow 0$  
as  $\alpha'\rightarrow \infty$ and  $\psi( x\T _{i}+\alpha\mathbf{1}_{m_x}\T) \rightarrow 0$ as $\alpha\rightarrow \infty$, we have that with  $\alpha,\alpha'$ sufficiently large, for any $i \in \{1,\dots,n\}$,
\begin{align*} 
\left|h^{(H)}_X(w_{(1:H)}) _{ii}   \right| > \sum_{k\neq i} \left|h^{(H)}_X(w_{(1:H)})_{ik}    \right|,
\end{align*}
which means that the matrix $[h^{(H)}_X(w_{(1:H)})_{ij} ]_{1\le i,j\le n} \in \RR^{n \times n}$
 is strictly diagonally dominant and hence is nonsingular with rank $n$.  Since the set of all columns of $h^{(H)}_X(w_{(1:H)})\in \RR^{n \times (m_H + 1)}$ contains all columns of $[h^{(H)}_X(w_{(1:H)})_{ij} ]_{1\le i,j\le n} \in \RR^{n \times n}$, this implies that $h^{(H)}_X(w_{(1:H)})$ has rank $n$ and $\varphi(w_{(1:H)})\neq 0$ for this constructed particular $w_{(1:H)}$. 
Since $\varphi(w_{(1:H)})$ is not identically zero,  the Lebesgue measure of the  set 
$
 \{w_{(1:H)}\in \RR^{d_{1:H}}:\varphi(w_{(1:H)})=0\}=\{w_{(1:H)}\in \RR^{d_{1:H}}: \rank(h_{X}^{(H)}(w_{(1:H)})) \neq n \}
$
is zero. 

\end{proof}

The following lemma is used to prove the global convergence for  the easier case of  $\min(m_1,\dots,m_{H-1})\ge  n$.
\begin{lemma} \label{lemma:rank_step1}
Suppose Assumption \ref{a:input} hold. Assume that  $H\ge1$ and $\min(m_1,\dots,m_H)\ge n$. Then, the Lebesgue measure of the  set 
$
 \{w_{(1:H)}\in \RR^{d_{1:H}}: \rank(h_{X}^{(H)}(w_{(1:H)})) \neq n \}
$
is zero.
\end{lemma}

\begin{proof}
Define $\varphi(w_{(1:H)})=\det([h_{X}^{(H)}(w_{(1:H)}), \mathbf{1}_{n}][h_{X}^{(H)}(w_{(1:H)}), \mathbf{1}_{n}]^{\top})$, which is   analytic  since $\sigma$ is analytic.
Furthermore, we have that $\{w_{(1:H)}\in \RR^{d_{1:H}}: \rank([h_{X}^{(H)}(w_{(1:H)}), \mathbf{1}_{n}]) \neq n \}
=\{w_{(1:H)}\in \RR^{d_{1:H}}:\varphi(w_{(1:H)})=0\},
$
since the rank of $[h_{X}^{(H)}(w_{(1:H)}), \mathbf{1}_{n}]$ and the rank of the Gram matrix are equal.
Since $\varphi$ is analytic, if $\varphi$ is not identically zero ($\varphi\neq 0$),  the Lebesgue measure of its zero set 
$
\{w_{(1:H)}\in \RR^{d_{1:H}}:\varphi(w_{(1:H)})=0\}
$
is zero \citep{mityagin2015zero}. 
Therefore, if   $\varphi(w_{(1:H)})\neq 0$ for some $w_{(1:H)}\in \RR^{d_{1:H}}$, the  Lebesgue measure of the set $\{w_{(1:H)}\in \RR^{d_{1:H}}: \text{$[h_{X}^{(H)}(w_{(1:H)}), \mathbf{1}_{n}]$ has rank less than ${n}$} \}$ is zero. 

Accordingly, we now constructs a $w_{(1:H)}\in \RR^{d_{1:H}}$such that $\varphi(w_{(1:H)})\neq 0$. 
From Assumption \ref{a:input}, there exists $c>0$ such that $\|x_i\|_2^2- x_i\T x_j > c$ for any $x_i,x_j \in S_{x}$ with $i\neq j$. Using such a $c$, set $W^{(1)}_i=\alpha^{(1)} x_i \in \RR^{m_x}$ and $b^{(1)}_i=\alpha^{(1)}(c/2 - \|x_{i}\|_2^2)$ for $i=1,\dots,n$. Moreover, set $W^{(l)}_{1:n,1:n}=\alpha^{(l)} I_{n}$ and $b^{(l)}_{1:n}=-\alpha^{(l)} \mathbf{1}_{n}$ for $l=2,\dots,H$.
Set all other weights and bias to be zero. 

Then,
for any $i\in\{1,\dots,n\}$,
$$
h^{(1)}_X(w_{(1:1)}) _{ii} = \sigma(c\alpha^{(1)}/2), 
$$
and for any $k\in\{1,\dots,n\}$ with  $k \neq i$,  
$$
h^{(1)}_X(w_{(1:1)})_{ik} =\sigma(\alpha^{(1)} (\langle x_i,x_k\rangle- \|x_k\|_2^{2}+c/2) ) \le\sigma(-c\alpha^{(1)}/2).  
$$
Since $\sigma(c\alpha^{(1)}/2)\rightarrow \infty$ and $\sigma(-c\alpha^{(1)}/2) \rightarrow 0$  
as  $\alpha^{(1)}\rightarrow \infty$, with  $\alpha^{(1)}$ sufficiently large, we have that $\sigma(c\alpha^{(1)}/2)-1\ge c^{(2)}_1$ and $\sigma(-c\alpha^{(1)}/2)-1 \le -c^{(2)}_2 $ for some $c^{(2)}_1,c^{(2)}_2>0$.
Therefore, with  $\alpha^{(1)}$ sufficiently large,
$$
h^{(2)}_X(w_{(1:2)}) _{ii} =\sigma(\alpha^{(2)} (\sigma(c\alpha^{(1)}/2)-1) ) \ge\sigma(\alpha^{(2)} c^{(2)}_1) ,
$$
and  
$$
h^{(2)}_X(w_{(1:2)})_{ik}  \le \sigma(\alpha^{(2)}(\sigma(-c\alpha^{(1)}/2)-1)  ) \le\sigma(-\alpha^{(2)} c^{(2)}_2)   . 
$$
Since $\sigma(\alpha^{(2)} c^{(2)}_1)\rightarrow \infty$ and $\sigma(-\alpha^{(2)} c^{(2)}_2) \rightarrow 0$  
as  $\alpha^{(2)}\rightarrow \infty$, with  $\alpha^{(2)}$ sufficiently large, we have that $\sigma(\alpha^{(2)} c^{(2)}_1)-1\ge c^{(3)}_1$ and $\sigma(-\alpha^{(2)} c^{(2)}_2)-1 \le -c^{(3)}_2 $ for some $c^{(3)}_1,c^{(3)}_2>0$. Note that given the $\alpha^{(1)}$, $c^{(3)}_1$ and $c^{(3)}_2$ depends only on $\alpha^{(2)}$ and does not depend on any of $\alpha^{(3)},\dots,\alpha^{(H)}$. Therefore, with  $\alpha^{(2)}$ sufficiently large,
$$
h^{(3)}_X(w_{(1:3)}) _{ii} \ge \sigma(\alpha^{(3)} (\sigma(\alpha^{(2)} c^{(2)}_1)-1) ) \ge\sigma(\alpha^{(3)} c^{(3)}_1) ,
$$
and  
$$
h^{(3)}_X(w_{(1:3)})_{ik}  \le \sigma(\alpha^{(3)}(\sigma(-\alpha^{(2)} c^{(2)}_2)-1)  ) \le\sigma(-\alpha^{(3)} c^{(3)}_2)   .
$$
Repeating this process for $l=1,\dots,H$, we have that with    $\alpha^{(1)},\dots,\alpha^{(H-1)}$ sufficiently large,
$$h^{(H)}_X(w_{(1:H)}) _{ii} \ge\sigma(\alpha^{(H)} c^{(H)}_1),$$ 
and 
$$
h^{(H)}_X(w_{(1:H)})_{ik} \le\sigma(-\alpha^{(H)} c^{(H)}_2),
$$
where $c^{(H)}_1=c^{(H)}_2=c/2$ if $H=1$.

Here,  $h^{(H)}_X(w_{(1:H)}) _{ii}\rightarrow \infty$ and $h^{(H)}_X(w_{(1:H)})_{ik} \rightarrow0$ as $\alpha^{(H)} \rightarrow \infty$.
Therefore, with    $\alpha^{(1)},\dots,\alpha^{(H)}$ sufficiently large,
for any $i \in \{1,\dots,n\}$,
\begin{align} \label{eq:lemma:rank_step1_1}
\left|h^{(H)}_X(w_{(1:H)}) _{ii}   \right| > \sum_{k\neq i} \left|h^{(H)}_X(w_{(1:H)})_{ik}    \right|.
\end{align}
The inequality \eqref{eq:lemma:rank_step1_1} means that  the matrix $[h^{(H)}_X(w_{(1:H)})_{ij} ]_{1\le i,j\le n} \in \RR^{n \times n}$
 is strictly diagonally dominant and hence is nonsingular with rank $n$.  Since the set of all columns of $h^{(H)}_X(w_{(1:H)})\in \RR^{n \times (m_H + 1)}$ contains all columns of $[h^{(H)}_X(w_{(1:H)})_{ij} ]_{1\le i,j\le n} \in \RR^{n \times n}$, this implies that $h^{(H)}_X(w_{(1:H)})$ has rank $n$ and $\varphi(w_{(1:H)})\neq 0$ for this constructed particular $w_{(1:H)}$. 
Since $\varphi(w_{(1:H)})$ is not identically zero,  the Lebesgue measure of the  set 
$
 \{w_{(1:H)}\in \RR^{d_{1:H}}:\varphi(w_{(1:H)})=0\}=\{w_{(1:H)}\in \RR^{d_{1:H}}: \rank(h_{X}^{(H)}(w_{(1:H)})) \neq n \}
$
is zero.

\end{proof}

Using the matrix notations, we obtain the following lemma: 
\begin{lemma} \label{lemma:f_manipulations}
For any $x$ and $w$, \begin{align*} 
f(x,w)\T = \left[I_{m_{y}} \otimes\begin{bmatrix} h^{(H)}(x_{i},w_{(1:H)}^{}) & 1 \\
\end{bmatrix} \right] w_{(H+1)} \in \RR^{m_y}.
\end{align*}
\end{lemma}
\begin{proof} This is a result of the following arithmetic manipulations:
\begin{align*} 
f(x,w)\T&=\vect(f(x,w))
\\ & =\vect \left( \begin{bmatrix} h^{(H)}(x_{i},w_{(1:H)}^{}) & 1 \\
\end{bmatrix}\begin{bmatrix} W^{(H+1)} \\
b^{(H+1)} \\
\end{bmatrix} I_{m_y} \right)
\\ &=\left[I_{m_{y}} \otimes\begin{bmatrix} h^{(H)}(x_{i},w_{(1:H)}^{}) & 1 \\
\end{bmatrix} \right] \vect \left(\begin{bmatrix} W^{(H+1)} \\
b^{(H+1)} \\
\end{bmatrix} \right)
\\ & = \left[I_{m_{y}} \otimes\begin{bmatrix} h^{(H)}(x_{i},w_{(1:H)}^{}) & 1 \\
\end{bmatrix} \right] w_{(H+1)}.
\end{align*}

\end{proof}

For the sake of completeness, we state a simple modification of a well known fact:

\begin{lemma} \label{lemma:known_1}
For any differentiable function $\varphi: \dom(\varphi) \rightarrow \RR$ with an open  convex domain $\dom(\varphi) \subseteq \RR^{n_\varphi}$, if   $\|\nabla \varphi(z') - \nabla \varphi(z)\| \le L \|z'-z\|$ for all $z,z' \in \dom(\varphi)$,
 then
 $$
\varphi(z') \le \varphi(z) + \nabla \varphi(z)\T (z'-z) + \frac{L}{2} \|z'-z\|^2 \quad   \text{for all $z,z' \in \dom(\varphi) $}.
$$
\end{lemma}
\begin{proof}
Fix $z,z'\in \dom(\varphi) \subseteq \RR^{d_{\varphi}}$. Since $\dom(\varphi)$ is a convex set,  $z+t(z'-z)\in\dom(\varphi)$ for all $t \in [0, 1]$.
Since $\dom(\varphi) $ is open, there exists $\epsilon >0$ such that $z+(1+\epsilon')(z'-z)\in\dom(\varphi)$ and $z+(0-\epsilon')(z'-z)\in\dom(\varphi)$ for all $\epsilon'\le \epsilon$. Fix $\epsilon>0$ to be such a number. Combining these, $z+t(z'-z)\in\dom(\varphi)$ for all $t \in [0-\epsilon, 1+\epsilon]$.

Accordingly, we can define  a function $\bar \varphi: [0-\epsilon, 1+\epsilon] \rightarrow \RR$ by $\bar \varphi(t)=\varphi(z+t(z'-z))$. Then, $\bar \varphi(1)=\varphi(z')$, $\bar \varphi(0)=\varphi(z)$, and $\nabla \bar \varphi(t)=\nabla\varphi(z+t(z'-z))\T (z'-z)$ for $t \in [0, 1] \subset (0-\epsilon,1+\epsilon)$. Since $\|\nabla \varphi(z') - \nabla \varphi(z)\| \le L \|z'-z\|$, \begin{align*}
\|\nabla\bar  \varphi(t')-\nabla\bar  \varphi(t) \| &=\|[\nabla\varphi(z+t'(z'-z)) -\nabla\varphi(z+t(z'-z))\T (z'-z) \|
\\ &\le \|z'-z\|\|\nabla\varphi(z+t'(z'-z)) -\nabla\varphi(z+t(z'-z))  \| \\ & \le  L\|z'-z\|\|(t'-t)(z'-z)   \|
\\ & \le L\|z'-z\|^{2}\|t'-t   \|.
\end{align*}
Thus, $\nabla \bar \varphi:[0, 1]\rightarrow \RR$ is Lipschitz continuous with the Lipschitz constant $L\|z'-z\|^{2}$, and hence $\nabla\bar  \varphi$ is continuous. 

By using the fundamental theorem of calculus with the continuous function $\nabla\bar  \varphi:[0, 1]  \rightarrow \RR$,
\begin{align*}
\varphi(z')&=\varphi(z)+ \int_0^1 \nabla\varphi(z+t(z'-z))\T (z'-z)dt
\\ &=\varphi(z)+\nabla\varphi(z)\T (z'-z)+ \int_0^1 [\nabla\varphi(z+t(z'-z))-\nabla\varphi(z)]\T (z'-z)dt
\\ & \le \varphi(z)+\nabla\varphi(z)\T (z'-z)+ \int_0^1 \|\nabla\varphi(z+t(z'-z))-\nabla\varphi(z)\| \|z'-z \|dt
\\ & \le \varphi(z)+\nabla\varphi(z)\T (z'-z)+ \int_0^1 tL\|z'-z\|^{2}dt
\\ & =  \varphi(z)+\nabla\varphi(z)\T (z'-z)+\frac{L}{2}\|z'-z\|^{2}. 
\end{align*} 

\end{proof}

We also utilize the following  well known fact:
\begin{lemma} \label{lemma:known_2}
For any differentiable function $\varphi: \RR^{d_{\varphi}} \rightarrow \RR$ and  any $z^{t+1},z^t\in \RR^{d_{\varphi}}$ such that  $z^{t+1}=z^t- \frac{1}{L}\nabla \varphi(z^t)$, the following holds:
$$
\nabla \varphi(z^{t})\T (z-z^{t}) + \frac{L}{2} \|z-z^{t}\|^2 - \frac{L}{2} \|z-z^{t+1}\|^2 =\nabla \varphi(z^{t})\T (z^{t+1}-z^{t})+\frac{L}{2} \|z^{t+1}-z^{t}\|^2   \quad   \text{for all $z\in \RR^d$}.
$$
\end{lemma}
\begin{proof}
Using $z^{t+1}=z^t- \frac{1}{L}\nabla \varphi(z^t)$, which implies $\nabla \varphi(z^t)=L(z^t- z^{t+1})$, 
\begin{align*}
&\nabla \varphi(z^{t})\T (z-z^{t}) + \frac{L}{2} \|z-z^{t}\|^2 - \frac{L}{2} \|z-z^{t+1}\|^2 
\\ &=L(z^t- z^{t+1})\T(z-z^{t})   + \frac{L}{2} \|z-z^{t}\|^2 - \frac{L}{2} \|z-z^{t+1}\|^2 
\\ & =L\left( z\T\ z^t - \|z^t\|^2-z\T z^{t+1}+ (z^{t})\T z^{t+1}+ \frac{1}{2} \|z\|^2+ \frac{1}{2} \|z^{t}\|^2 -z\T z^t- \frac{1}{2} \|z\|^2-\frac{1}{2} \|z^{t+1}\|^2+z\T z^{t+1} \right) 
\\ &=L\left(  (z^{t})\T z^{t+1}- \frac{1}{2} \|z^{t}\|^2 -\frac{1}{2} \|z^{t+1}\|^2 \right)
\\ & = - \frac{L}{2} \|z^{t+1} - z^{t}\|^{2}
\\ & = -L(z^{t+1} - z^{t})\T (z^{t+1} - z^{t})+ \frac{L}{2}   \|z^{t+1} - z^{t}\|^{2}
\\ & = \nabla \varphi(z^t)(z^{t+1} - z^{t})+ \frac{L}{2}   \|z^{t+1} - z^{t}\|^{2}.
\end{align*}

\end{proof}

 With those lemmas, we are now ready to complete a proof for Theorem \ref{thm:general} for fully connected networks.
From Lemmas \ref{lemma:rank_step2} and \ref{lemma:rank_step1}, the Lebesgue measure of the  set 
$
 \{w_{(1:H)}\in \RR^{d_{1:H}}: \rank(h_{X}^{(H)}(w_{(1:H)})) \neq n \}
$
is zero. In Algorithm \ref{al:train}, $w_{(1:H)}^{\tau} \leftarrow w_{(1:H)}^{\tau}+\delta$ defines a non-degenerate Gaussian measure with the mean shifted by the original $w_{(1:H)}^{\tau}$. Since a non-degenerate Gaussian measure
with any mean and variance is absolutely continuous with respect to Lebesgue measure, with probability one, $\rank(h_{X}^{(H)}(w_{(1:H)}^{\tau}))=n$.
Since $f_X(w)_{k} = [h_{X}^{(H)}(w_{(1:H)}), \mathbf{1}_{n}] \begin{bmatrix} W^{(H+1)}_k \\  b^{(H+1)}_k \end{bmatrix}\in \RR^{n}$ and $w_{(H+1)}=\vect\left(\begin{bmatrix}W^{(H+1)} \\ b^{(H+1)} \\ \end{bmatrix} \right)$, with probability one, $$
\{f_X(w) \in \RR^{n \times m_y}:w\in \RR^d \} =\{f_X([(w_{(1:H)}^{\tau})\T,(w_{(H+1)})\T]\T) \in \RR^{n \times m_y}:w_{(H+1)}\in \RR^d \}.
$$
Thus,  for any global minimum   $w^{*} \in \RR^d$, there exists $w_{(H+1)}^{*}$ such that $f_X([(w_{(1:H)}^{\tau})\T,(w_{(H+1)}^{*})\T]\T)=f_X(w^{*})$, and hence
\begin{align}\label{eq:step_2_theorem_0}
 \Lcal([(w_{(1:H)}^{\tau})\T,(w_{(H+1)}^{*})\T]\T)=\Lcal_{(H+1)}(w_{(H+1)}^{*})=\Lcal(w^*),
\end{align}
where    $\Lcal_{(H+1)}(z)=\Lcal([(w_{(1:H)}^{\tau})\T,z\T]\T)$. 

Let $w^{*}$ be arbitrary, and $w_{(H+1)}^*$ be a corresponding vector that satisfies \eqref{eq:step_2_theorem_0}. Using Lemma \ref{lemma:f_manipulations},
$$
\Lcal_{(H+1)}^{}(w_{(H+1)}) = \frac{1}{n} \sum_{i=1}^n \ell_{i} \left( \left[I_{m_{y}} \otimes[h^{(H)}(x_{i},w_{(1:H)}^{\tau}), 1] \right] w_{(H+1)} \right), 
$$
Here, $\Lcal_{(H+1)}^{}$ is differentiable since $\ell_{i}$ is differentiable  (Assumption \ref{a:loss}). Furthermore,  
$$
\nabla_{w^{}_{(H+1)}} \Lcal(w^{})=\nabla \Lcal_{(H+1)}^{}(w_{(H+1)})=\frac{1}{n} \sum_{i=1}^n [I_{m_{y}} \otimes M_{i} ] \T\nabla\ell_{i} \left([I_{m_{y}} \otimes M_{i} ]w_{(H+1)}\right),
$$
where $M_{i}= [h^{(H)}(x_{i},w_{(1:H)}^{\tau}), 1]$. Using Assumption \ref{a:loss} and $\|[I_{m_{y}} \otimes M_{i} ]\|_2 = \|M_{i}\|_2$, for any $w_{(H+1)}',w_{(H+1)}$,
\begin{align*}
\|\nabla \Lcal_{(H+1)}^{}(w_{(H+1)}')-\nabla \Lcal_{(H+1)}^{}(w_{(H+1)})\| & \le \frac{1}{n}  \sum_{i=1}^n \|M_{i} \|_{2} \|\nabla\ell_{i} (M _{i}w'_{(H+1)}) -\nabla\ell_{i} \left(M_{i} w_{(H+1)}\right) \|_{2} 
\\ & \le \left(\frac{L_{\ell}}{n}  \sum_{i=1}^n \|M_{i}\|_2^2 \right) \|w'_{(H+1)}- w_{(H+1)} \|.
\end{align*}

\begin{proof}[Proof of Theorem \ref{thm:general} (i) for fully-connected  networks]

Let $t > \tau$. Using Lemma \ref{lemma:known_1} for the differentiable function $ \Lcal_{(H+1)}^{}$ with $L= \frac{L_{\ell}}{n}  \sum_{i=1}^n \|M_{i}\|_2^2=L_H$,
\begin{align} \label{eq:step_2_theorem_1}
\Lcal_{(H+1)}^{}(w_{(H+1)}^{t+1})\le \Lcal_{(H+1)}^{}(w_{(H+1)}^{t})+ \nabla  \Lcal_{(H+1)}^{}(w_{(H+1)}^{t})\T (w_{(H+1)}^{t+1}-w_{(H+1)}^{t}) + \frac{L_H}{2} \|w_{(H+1)}^{t+1}-w_{(H+1)}^{t}\|^2 .
\end{align}
Using \eqref{eq:step_2_theorem_1} and $\nabla  \Lcal_{(H+1)}^{}(w_{(H+1)}^{t})=L_{H}(w_{(H+1)}^t-w_{(H+1)}^{t+1}) $ (from $w_{(H+1)}^{t+1}=w_{(H+1)}^t-\frac{1}{L_H} \nabla  \Lcal_{(H+1)}^{}(w_{(H+1)}^{t})$),
\begin{align} \label{eq:step_2_theorem_2}
\nonumber \Lcal_{(H+1)}^{}(w_{(H+1)}^{t+1}) & \le \Lcal_{(H+1)}^{}(w_{(H+1)}^{t})- L_{H} \|w_{(H+1)}^{t+1}-w_{(H+1)}^{t}\|^2 +\frac{L_H}{2} \|w_{(H+1)}^{t+1}-w_{(H+1)}^{t}\|^2
\\\nonumber  & =\Lcal_{(H+1)}^{}(w_{(H+1)}^{t})- \frac{L_H}{2} \|w_{(H+1)}^{t+1}-w_{(H+1)}^{t}\|^2 \\ & \le\Lcal_{(H+1)}^{}(w_{(H+1)}^{t}),
 \end{align}
which shows that $\Lcal_{(H+1)}^{}(w_{(H+1)}^{t})$ is non-increasing in $t$.  
Using \eqref{eq:step_2_theorem_1} and Lemma \ref{lemma:known_2}, for any $z\in \RR^{d_{H+1}}$, 
\begin{align}  \label{eq:step_2_theorem_3}
&\Lcal_{(H+1)}^{}(w_{(H+1)}^{t+1})
\\ \nonumber &\le \Lcal_{(H+1)}^{}(w_{(H+1)}^{t})+ \nabla  \Lcal_{(H+1)}^{}(w_{(H+1)}^{t})\T (z^{}-w_{(H+1)}^{t}) + \frac{L_H}{2} \|z^{}-w_{(H+1)}^{t}\|^2 -\frac{L_H}{2} \|z-w_{(H+1)}^{t+1}\|^2 .  
\end{align}
Using \eqref{eq:step_2_theorem_3} and the facts that    $\ell_{i}$ is convex (Assumption \ref{a:loss}) and that $\Lcal_{(H+1)}$ is a nonnegative sum of the compositions of  $\ell_{i}$ and the affine map\ $w_{(H+1)} \mapsto [I_{m_{y}} \otimes M_{i} ]w_{(H+1)}$, we have that for any $z\in \RR^{d_{H+1}}$, 
\begin{align} \label{eq:step_2_theorem_4}
\Lcal_{(H+1)}^{}(w_{(H+1)}^{t+1}) \le\Lcal_{(H+1)}^{}(z)+ \frac{L_H}{2} \|z^{}-w_{(H+1)}^{t}\|^2 -\frac{L_H}{2} \|z-w_{(H+1)}^{t+1}\|^2 . 
\end{align}
Summing up both sides of \eqref{eq:step_2_theorem_4} and using \eqref{eq:step_2_theorem_2}, \begin{align*} 
 (t-\tau) \Lcal_{(H+1)}^{}(w_{(H+1)}^{t})\le\sum_{k=\tau}^{t-1} \Lcal_{(H+1)}(w_{(H+1)}^{k+1}) \le (t-\tau)\Lcal_{(H+1)}^{}(z)+ \frac{L_H}{2} \|z^{}-w_{(H+1)}^{\tau}\|^2 -\frac{L_H}{2} \|z-w_{(H+1)}^{t}\|^2,   
\end{align*}
which implies that for any $z\in \RR^{d_{H+1}}$, 
\begin{align*} 
\Lcal_{(H+1)}^{}(w_{(H+1)}^{t})\le \Lcal_{(H+1)}^{}(z)+ \frac{L_H\|z^{}-w_{(H+1)}^{\tau}\|^2}{2(t-\tau)}.
\end{align*}
Setting $z=w_{(H+1)}^*$ and using \eqref{eq:step_2_theorem_0}, 
$$
\Lcal_{(H+1)}^{}(w_{(H+1)}^{t})\le \Lcal_{(H+1)}^{}(w_{(H+1)}^{*})+ \frac{L_H\|w_{(H+1)}^{*}-w_{(H+1)}^{\tau}\|^2}{2(t-\tau)}=\Lcal(w^*)+ \frac{L_H\|w_{(H+1)}^{*}-w_{(H+1)}^{\tau}\|^2}{2(t-\tau)}.
$$

\end{proof}

\begin{proof}[Proof of Theorem \ref{thm:general} (ii)  for fully-connected  networks]
Let $t > \tau$. Using the assumptions of SGD, with probability one, 
\begin{align*}
\EE[\|w_{(H+1)}^{t+1}-w_{(H+1)}^{*}\| ^{2}\mid w^{t}]&=\EE[\|w_{(H+1)}^{t}-w_{(H+1)}^{*}-\bar \eta_t g^t \| ^{2}\mid w^{t}]
\\ & =\|w_{(H+1)}^{t}-w_{(H+1)}^{*}\|^{2}  -2\bar \eta_t \EE[ g^t \mid w^{t}]\T (w_{(H+1)}^{t}-w_{(H+1)}^{*})+\bar \eta_t^{2}\EE[ \|g^t\|^2 \mid w^{t}] 
\\ & \le\|w_{(H+1)}^{t}-w_{(H+1)}^{*}\|^{2}  -2\bar \eta_t (\Lcal_{(H+1)}^{}(w_{(H+1)}^{t})-\Lcal_{(H+1)}^{}(w_{(H+1)}^{*}))+\bar \eta_t^{2}\EE[ \|g^t\|^2 \mid w^{t}], 
\end{align*}
where the last line follows from the facts that    $\ell_{i}$ is convex (Assumption \ref{a:loss}) and that $\Lcal_{(H+1)}$ is a nonnegative sum of the compositions of  $\ell_{i}$ and the affine map\ $w_{(H+1)} \mapsto [I_{m_{y}} \otimes M_{i} ]w_{(H+1)}$. Taking expectation over $w^{t}$,
\begin{align*}
\EE[\|w_{(H+1)}^{t+1}-w_{(H+1)}^{*}\| ^{2}] \le \EE[\|w_{(H+1)}^{t}-w_{(H+1)}^{*}\|^{2}]-2\bar \eta_t (\EE[\Lcal_{(H+1)}^{}(w_{(H+1)}^{t})]-\Lcal_{(H+1)}^{}(w_{(H+1)}^{*}))+\bar \eta_t^{2}G^{2}.  
\end{align*}
By recursively applying this inequality over $t$,
\begin{align*}
\EE[\|w_{(H+1)}^{t+1}-w_{(H+1)}^{*}\| ^{2}] \le\EE[\|w_{(H+1)}^{\tau}-w_{(H+1)}^{*}\| ^{2}]-2 \sum_{k=\tau}^t \bar \eta_{k} (\EE[\Lcal_{(H+1)}^{}(w_{(H+1)}^{k})]-\Lcal_{(H+1)}^{}(w_{(H+1)}^{*}))+G^{2}  \sum_{k=\tau}^t\bar \eta_{k}^{2}.  
\end{align*}
Since $\|w_{(H+1)}^{t+1}-w_{(H+1)}^{*}\| ^{2} \ge 0$,
 \begin{align*}
 2 \sum_{k=\tau}^t \bar \eta_{k}  \EE[\Lcal_{(H+1)}^{}(w_{(H+1)}^{k})] \le \left(2 \sum_{k=\tau}^t  \bar \eta_{k} \right) \Lcal_{(H+1)}^{}(w_{(H+1)}^{*})+ R^{2}+G^{2}  \sum_{k=\tau}^t\bar \eta_{k}^{2}.
\end{align*} 
 Using \eqref{eq:step_2_theorem_0},  
\begin{align*}
 \min_{k=\tau,\tau+1,\dots, t}\EE[\Lcal_{}^{}(w_{}^{k})]   \le  \Lcal_{}^{}(w_{}^{*})+ \frac{R^{2}+G^{2}  \sum_{k=\tau}^t\bar \eta_{k}^{2}}{ 2 \sum_{k=\tau}^t \bar \eta_{k}}.
\end{align*} 
Using Jensen's inequality and the concavity of the minimum function,  
\begin{align*}
 \EE\left[\min_{k=\tau,\tau+1,\dots, t}\Lcal_{}(w_{}^{k})] \right]  \le  \Lcal(w^{*})+ \frac{R^{2}+G^{2}  \sum_{k=\tau}^t\bar \eta_{k}^{2}}{ 2 \sum_{k=\tau}^t \bar \eta_{k}}.
\end{align*} 
\end{proof}

\section{Proof of Theorem \ref{thm:general}}

By  building upon the  proof of Theorem \ref{thm:general} for fully connected layer from the previous section, we complete the  proof of Theorem \ref{thm:general} for the general case. First, in the following lemma,  we show that if the expressivity condition (Assumption  \ref{a:existance_of_w}) holds, then the random perturbation would generate a full rank hidden layer output with probability one.  

\begin{lemma} \label{lemma:rank_step3}
\emph{}
 Suppose Assumption  \ref{a:existance_of_w} hold. Assume that  $H\ge1$ and $m_H \ge n$. Then, the Lebesgue measure of the  set 
$
 \{w_{(1:H)}\in \RR^{d_{1:H}}: \rank(\allowbreak[h_{X}^{(H)}(w_{(1:H)}), \mathbf{1}_{n}]) \neq n \}
$
is zero.
\end{lemma}
\begin{proof} 

Define $\varphi(w_{(1:H)})= \det(\allowbreak[h_{X}^{(H)}(w_{(1:H)}), \mathbf{1}_{n}][h_{X}^{(H)}(w_{(1:H)}), \mathbf{1}_{n}]^{\top})$, which is   real analytic  since $h^{(H)}(x,w_{(1:H)})$ is real analytic.
Furthermore, we have that $\{w_{(1:H)}\in \RR^{d_{1:H}}: \rank(\allowbreak[h_{X}^{(H)}(w_{(1:H)}), \mathbf{1}_{n}]) \neq n \}
=\{w_{(1:H)}\in \RR^{d_{1:H}}:\varphi(w_{(1:H)})=0\},
$
since the rank of $\allowbreak[h_{X}^{(H)}(w_{(1:H)}), \mathbf{1}_{n}]$ and the rank of the Gram matrix are equal.
Since $\varphi$ is real analytic, if $\varphi$ is not identically zero ($\varphi\neq 0$),  the Lebesgue measure of its zero set 
$
\{w_{(1:H)}\in \RR^{d_{1:H}}:\varphi(w_{(1:H)})=0\}
$
is zero \citep{mityagin2015zero}. From Assumption \ref{a:existance_of_w},
   $\varphi(w_{(1:H)})\neq 0$ for some $w_{(1:H)}\in \RR^{d_{1:H}}$, and hence the  Lebesgue measure of the set $\{w_{(1:H)}\in \RR^{d_{1:H}}: \text{$\rank(\allowbreak[h_{X}^{(H)}(w_{(1:H)}), \mathbf{1}_{n}]) \neq n$} \}$ is zero. 
\end{proof}
 From Lemma \ref{lemma:rank_step3}, the Lebesgue measure of the  set 
$
 \{w_{(1:H)}\in \RR^{d_{1:H}}: \rank(h_{X}^{(H)}(w_{(1:H)})) \neq n \}
$
is zero. In Algorithm \ref{al:train}, $w_{(1:H)}^{\tau} \leftarrow w_{(1:H)}^{\tau}+\delta$ defines a non-degenerate Gaussian measure with the mean shifted by the original $w_{(1:H)}^{\tau}$. Since a non-degenerate Gaussian measure
with any mean and variance is absolutely continuous with respect to Lebesgue measure, with probability one, $\rank(h_{X}^{(H)}(w_{(1:H)}^{\tau}))=n$.
Since $f_X(w)_{k} = [h_{X}^{(H)}(w_{(1:H)}), \mathbf{1}_{n}] \begin{bmatrix} W^{(H+1)}_k \\  b^{(H+1)}_k \end{bmatrix}\in \RR^{n}$ and $w_{(H+1)}=\vect\left(\begin{bmatrix}W^{(H+1)} \\ b^{(H+1)} \\ \end{bmatrix} \right)$, with probability one, $$
\{f_X(w) \in \RR^{n \times m_y}:w\in \RR^d \} =\{f_X([(w_{(1:H)}^{\tau})\T,(w_{(H+1)})\T]\T) \in \RR^{n \times m_y}:w_{(H+1)}\in \RR^d \}.
$$
Thus,  for any global minimum   $w^{*} \in \RR^d$, there exists $w_{(H+1)}^{*}$ such that $f_X([(w_{(1:H)}^{\tau})\T,(w_{(H+1)}^{*})\T]\T)=f_X(w^{*})$, and hence
\begin{align}\label{eq:step_2_theorem_0_2}
 \Lcal([(w_{(1:H)}^{\tau})\T,(w_{(H+1)}^{*})\T]\T)=\Lcal_{(H+1)}(w_{(H+1)}^{*})=\Lcal(w^*),
\end{align}
where    $\Lcal_{(H+1)}(z)=\Lcal([(w_{(1:H)}^{\tau})\T,z\T]\T)$. 

Let $w^{*}$ be arbitrary, and $w_{(H+1)}^*$ be a corresponding vector that satisfies \eqref{eq:step_2_theorem_0_2}. Using Lemma \ref{lemma:f_manipulations},
$$
\Lcal_{(H+1)}^{}(w_{(H+1)}) = \frac{1}{n} \sum_{i=1}^n \ell_{i} \left( \left[I_{m_{y}} \otimes[h^{(H)}(x_{i},w_{(1:H)}^{\tau}), 1] \right] w_{(H+1)} \right), 
$$
Here, $\Lcal_{(H+1)}^{}$ is differentiable since $\ell_{i}$ is differentiable  (Assumption \ref{a:loss}). Furthermore,  
$$
\nabla_{w^{}_{(H+1)}} \Lcal(w^{})=\nabla \Lcal_{(H+1)}^{}(w_{(H+1)})=\frac{1}{n} \sum_{i=1}^n [I_{m_{y}} \otimes M_{i} ] \T\nabla\ell_{i} \left([I_{m_{y}} \otimes M_{i} ]w_{(H+1)}\right),
$$
where $M_{i}= [h^{(H)}(x_{i},w_{(1:H)}^{\tau}), 1]$. Using Assumption \ref{a:loss} and $\|[I_{m_{y}} \otimes M_{i} ]\|_2 = \|M_{i}\|_2$, for any $w_{(H+1)}',w_{(H+1)}$,
\begin{align*}
\|\nabla \Lcal_{(H+1)}^{}(w_{(H+1)}')-\nabla \Lcal_{(H+1)}^{}(w_{(H+1)})\| & \le \frac{1}{n}  \sum_{i=1}^n \|M_{i} \|_{2} \|\nabla\ell_{i} (M _{i}w'_{(H+1)}) -\nabla\ell_{i} \left(M_{i} w_{(H+1)}\right) \|_{2} 
\\ & \le \left(\frac{L_{\ell}}{n}  \sum_{i=1}^n \|M_{i}\|_2^2 \right) \|w'_{(H+1)}- w_{(H+1)} \|.
\end{align*}

\begin{proof}[Proof of Theorem \ref{thm:general} (i)]

Let $t > \tau$. Using Lemma \ref{lemma:known_1} for the differentiable function $ \Lcal_{(H+1)}^{}$ with $L= \frac{L_{\ell}}{n}  \sum_{i=1}^n \|M_{i}\|_2^2=L_H$,
\begin{align} \label{eq:step_2_theorem_1_2}
\Lcal_{(H+1)}^{}(w_{(H+1)}^{t+1})\le \Lcal_{(H+1)}^{}(w_{(H+1)}^{t})+ \nabla  \Lcal_{(H+1)}^{}(w_{(H+1)}^{t})\T (w_{(H+1)}^{t+1}-w_{(H+1)}^{t}) + \frac{L_H}{2} \|w_{(H+1)}^{t+1}-w_{(H+1)}^{t}\|^2 .
\end{align}
Using \eqref{eq:step_2_theorem_1_2} and $\nabla  \Lcal_{(H+1)}^{}(w_{(H+1)}^{t})=L_{H}(w_{(H+1)}^t-w_{(H+1)}^{t+1}) $ (from $w_{(H+1)}^{t+1}=w_{(H+1)}^t-\frac{1}{L_H} \nabla  \Lcal_{(H+1)}^{}(w_{(H+1)}^{t})$),
\begin{align} \label{eq:step_2_theorem_2_2}
\nonumber \Lcal_{(H+1)}^{}(w_{(H+1)}^{t+1}) & \le \Lcal_{(H+1)}^{}(w_{(H+1)}^{t})- L_{H} \|w_{(H+1)}^{t+1}-w_{(H+1)}^{t}\|^2 +\frac{L_H}{2} \|w_{(H+1)}^{t+1}-w_{(H+1)}^{t}\|^2
\\\nonumber  & =\Lcal_{(H+1)}^{}(w_{(H+1)}^{t})- \frac{L_H}{2} \|w_{(H+1)}^{t+1}-w_{(H+1)}^{t}\|^2 \\ & \le\Lcal_{(H+1)}^{}(w_{(H+1)}^{t}),
 \end{align}
which shows that $\Lcal_{(H+1)}^{}(w_{(H+1)}^{t})$ is non-increasing in $t$.  
Using \eqref{eq:step_2_theorem_1_2} and Lemma \ref{lemma:known_2}, for any $z\in \RR^{d_{H+1}}$, 
\begin{align}  \label{eq:step_2_theorem_3_2}
&\Lcal_{(H+1)}^{}(w_{(H+1)}^{t+1})
\\ \nonumber &\le \Lcal_{(H+1)}^{}(w_{(H+1)}^{t})+ \nabla  \Lcal_{(H+1)}^{}(w_{(H+1)}^{t})\T (z^{}-w_{(H+1)}^{t}) + \frac{L_H}{2} \|z^{}-w_{(H+1)}^{t}\|^2 -\frac{L_H}{2} \|z-w_{(H+1)}^{t+1}\|^2 .  
\end{align}
Using \eqref{eq:step_2_theorem_3_2} and the facts that    $\ell_{i}$ is convex (Assumption \ref{a:loss}) and that $\Lcal_{(H+1)}$ is a nonnegative sum of the compositions of  $\ell_{i}$ and the affine map\ $w_{(H+1)} \mapsto [I_{m_{y}} \otimes M_{i} ]w_{(H+1)}$, we have that for any $z\in \RR^{d_{H+1}}$, 
\begin{align} \label{eq:step_2_theorem_4_2}
\Lcal_{(H+1)}^{}(w_{(H+1)}^{t+1}) \le\Lcal_{(H+1)}^{}(z)+ \frac{L_H}{2} \|z^{}-w_{(H+1)}^{t}\|^2 -\frac{L_H}{2} \|z-w_{(H+1)}^{t+1}\|^2 . 
\end{align}
Summing up both sides of \eqref{eq:step_2_theorem_4_2} and using \eqref{eq:step_2_theorem_2_2}, \begin{align*} 
 (t-\tau) \Lcal_{(H+1)}^{}(w_{(H+1)}^{t})\le\sum_{k=\tau}^{t-1} \Lcal_{(H+1)}(w_{(H+1)}^{k+1}) \le (t-\tau)\Lcal_{(H+1)}^{}(z)+ \frac{L_H}{2} \|z^{}-w_{(H+1)}^{\tau}\|^2 -\frac{L_H}{2} \|z-w_{(H+1)}^{t}\|^2,   
\end{align*}
which implies that for any $z\in \RR^{d_{H+1}}$, 
\begin{align*} 
\Lcal_{(H+1)}^{}(w_{(H+1)}^{t})\le \Lcal_{(H+1)}^{}(z)+ \frac{L_H\|z^{}-w_{(H+1)}^{\tau}\|^2}{2(t-\tau)}.
\end{align*}
Setting $z=w_{(H+1)}^*$ and using \eqref{eq:step_2_theorem_0_2}, 
$$
\Lcal_{(H+1)}^{}(w_{(H+1)}^{t})\le \Lcal_{(H+1)}^{}(w_{(H+1)}^{*})+ \frac{L_H\|w_{(H+1)}^{*}-w_{(H+1)}^{\tau}\|^2}{2(t-\tau)}=\Lcal(w^*)+ \frac{L_H\|w_{(H+1)}^{*}-w_{(H+1)}^{\tau}\|^2}{2(t-\tau)}.
$$

\end{proof}

\begin{proof}[Proof of Theorem \ref{thm:general} (ii)]
Let $t > \tau$. Using the assumptions of SGD, with probability one, 
\begin{align*}
\EE[\|w_{(H+1)}^{t+1}-w_{(H+1)}^{*}\| ^{2}\mid w^{t}]&=\EE[\|w_{(H+1)}^{t}-w_{(H+1)}^{*}-\bar \eta_t g^t \| ^{2}\mid w^{t}]
\\ & =\|w_{(H+1)}^{t}-w_{(H+1)}^{*}\|^{2}  -2\bar \eta_t \EE[ g^t \mid w^{t}]\T (w_{(H+1)}^{t}-w_{(H+1)}^{*})+\bar \eta_t^{2}\EE[ \|g^t\|^2 \mid w^{t}] 
\\ & \le\|w_{(H+1)}^{t}-w_{(H+1)}^{*}\|^{2}  -2\bar \eta_t (\Lcal_{(H+1)}^{}(w_{(H+1)}^{t})-\Lcal_{(H+1)}^{}(w_{(H+1)}^{*}))+\bar \eta_t^{2}\EE[ \|g^t\|^2 \mid w^{t}], 
\end{align*}
where the last line follows from the facts that    $\ell_{i}$ is convex (Assumption \ref{a:loss}) and that $\Lcal_{(H+1)}$ is a nonnegative sum of the compositions of  $\ell_{i}$ and the affine map\ $w_{(H+1)} \mapsto [I_{m_{y}} \otimes M_{i} ]w_{(H+1)}$. Taking expectation over $w^{t}$,
\begin{align*}
\EE[\|w_{(H+1)}^{t+1}-w_{(H+1)}^{*}\| ^{2}] \le \EE[\|w_{(H+1)}^{t}-w_{(H+1)}^{*}\|^{2}]-2\bar \eta_t (\EE[\Lcal_{(H+1)}^{}(w_{(H+1)}^{t})]-\Lcal_{(H+1)}^{}(w_{(H+1)}^{*}))+\bar \eta_t^{2}G^{2}.  
\end{align*}
By recursively applying this inequality over $t$,
\begin{align*}
\EE[\|w_{(H+1)}^{t+1}-w_{(H+1)}^{*}\| ^{2}] \le\EE[\|w_{(H+1)}^{\tau}-w_{(H+1)}^{*}\| ^{2}]-2 \sum_{k=\tau}^t \bar \eta_{k} (\EE[\Lcal_{(H+1)}^{}(w_{(H+1)}^{k})]-\Lcal_{(H+1)}^{}(w_{(H+1)}^{*}))+G^{2}  \sum_{k=\tau}^t\bar \eta_{k}^{2}.  
\end{align*}
Since $\|w_{(H+1)}^{t+1}-w_{(H+1)}^{*}\| ^{2} \ge 0$,
 \begin{align*}
 2 \sum_{k=\tau}^t \bar \eta_{k}  \EE[\Lcal_{(H+1)}^{}(w_{(H+1)}^{k})] \le \left(2 \sum_{k=\tau}^t  \bar \eta_{k} \right) \Lcal_{(H+1)}^{}(w_{(H+1)}^{*})+ R^{2}+G^{2}  \sum_{k=\tau}^t\bar \eta_{k}^{2}.
\end{align*} 
 Using \eqref{eq:step_2_theorem_0_2},  
\begin{align*}
 \min_{k=\tau,\tau+1,\dots, t}\EE[\Lcal_{}^{}(w_{}^{k})]   \le  \Lcal_{}^{}(w_{}^{*})+ \frac{R^{2}+G^{2}  \sum_{k=\tau}^t\bar \eta_{k}^{2}}{ 2 \sum_{k=\tau}^t \bar \eta_{k}}.
\end{align*} 
Using Jensen's inequality and the concavity of the minimum function,  
\begin{align*}
 \EE\left[\min_{k=\tau,\tau+1,\dots, t}\Lcal_{}(w_{}^{k})] \right]  \le  \Lcal(w^{*})+ \frac{R^{2}+G^{2}  \sum_{k=\tau}^t\bar \eta_{k}^{2}}{ 2 \sum_{k=\tau}^t \bar \eta_{k}}.
\end{align*} 
\end{proof}

\section{Proof of Theorem \ref{thm:fcnet}}
Although we have prove the global convergence for fully-connected networks in the first section, here we prove theorem \ref{thm:fcnet} explicitly to show all the conditions for global convergence are satisfied. 
\begin{proof}
Let us first focus on the case of   $\min(m_1,\dots,m_{H-1})\ge  n$. 
From Assumption \ref{a:input}, there exists $c>0$ such that $\|x_i\|_2^2- x_i\T x_j > c$ for any $x_i,x_j \in S_{x}$ with $i\neq j$. Using such a $c$, set $W^{(1)}_i=\alpha^{(1)} x_i \in \RR^{m_x}$ and $b^{(1)}_i=\alpha^{(1)}(c/2 - \|x_{i}\|_2^2)$ for $i=1,\dots,n$. Moreover, set $W^{(l)}_{1:n,1:n}=\alpha^{(l)} I_{n}$ and $b^{(l)}_{1:n}=-\alpha^{(l)} \mathbf{1}_{n}$ for $l=2,\dots,H$.
Set all other weights and bias to be zero. Then, for any $i\in\{1,\dots,n\}$,
$
h^{(1)}_X(w_{(1:1)}) _{ii} = \sigma(c\alpha^{(1)}/2), 
$
and for any $k\in\{1,\dots,n\}$ with  $k \neq i$,  
$
h^{(1)}_X(w_{(1:1)})_{ik} =\sigma(\alpha^{(1)} (\langle x_i,x_k\rangle- \|x_k\|_2^{2}+c/2) ) \le\sigma(-c\alpha^{(1)}/2).  
$
Since $\sigma(c\alpha^{(1)}/2)\rightarrow \infty$ and $\sigma(-c\alpha^{(1)}/2) \rightarrow 0$  
as  $\alpha^{(1)}\rightarrow \infty$, with  $\alpha^{(1)}$ sufficiently large, we have that $\sigma(c\alpha^{(1)}/2)-1\ge c^{(2)}_1$ and $\sigma(-c\alpha^{(1)}/2)-1 \le -c^{(2)}_2 $ for some $c^{(2)}_1,c^{(2)}_2>0$.
Therefore, with  $\alpha^{(1)}$ sufficiently large,
$
h^{(2)}_X(w_{(1:2)}) _{ii} =\sigma(\alpha^{(2)} (\sigma(c\alpha^{(1)}/2)-1) ) \ge\sigma(\alpha^{(2)} c^{(2)}_1) ,
$
and  
$
h^{(2)}_X(w_{(1:2)})_{ik}  \le \sigma(\alpha^{(2)}(\sigma(-c\alpha^{(1)}/2)-1)  ) \le\sigma(-\alpha^{(2)} c^{(2)}_2)   . 
$
Since $\sigma(\alpha^{(2)} c^{(2)}_1)\rightarrow \infty$ and $\sigma(-\alpha^{(2)} c^{(2)}_2) \rightarrow 0$  
as  $\alpha^{(2)}\rightarrow \infty$, with  $\alpha^{(2)}$ sufficiently large, we have that $\sigma(\alpha^{(2)} c^{(2)}_1)-1\ge c^{(3)}_1$ and $\sigma(-\alpha^{(2)} c^{(2)}_2)-1 \le -c^{(3)}_2 $ for some $c^{(3)}_1,c^{(3)}_2>0$. Note that given the $\alpha^{(1)}$, $c^{(3)}_1$ and $c^{(3)}_2$ depends only on $\alpha^{(2)}$ and does not depend on any of $\alpha^{(3)},\dots,\alpha^{(H)}$. 

Therefore, with  $\alpha^{(2)}$ sufficiently large,
$
h^{(3)}_X(w_{(1:3)}) _{ii} \ge \sigma(\alpha^{(3)} (\sigma(\alpha^{(2)} c^{(2)}_1)-1) ) \ge\sigma(\alpha^{(3)} c^{(3)}_1) ,
$
and  
$
h^{(3)}_X(w_{(1:3)})_{ik}  \le \sigma(\alpha^{(3)}(\sigma(-\alpha^{(2)} c^{(2)}_2)-1)  ) \le\sigma(-\alpha^{(3)} c^{(3)}_2)   .
$
Repeating this process for $l=1,\dots,H$, we have that with    $\alpha^{(1)},\dots,\alpha^{(H-1)}$ sufficiently large,
$$h^{(H)}_X(w_{(1:H)}) _{ii} \ge\sigma(\alpha^{(H)} c^{(H)}_1),$$ 
and 
$$
h^{(H)}_X(w_{(1:H)})_{ik} \le\sigma(-\alpha^{(H)} c^{(H)}_2),
$$
where $c^{(H)}_1=c^{(H)}_2=c/2$ if $H=1$.
Here,  $h^{(H)}_X(w_{(1:H)}) _{ii}\rightarrow \infty$ and $h^{(H)}_X(w_{(1:H)})_{ik} \rightarrow0$ as $\alpha^{(H)} \rightarrow \infty$.
Therefore, with    $\alpha^{(1)},\dots,\alpha^{(H)}$ sufficiently large,
for any $i \in \{1,\dots,n\}$,
$
\left|h^{(H)}_X(w_{(1:H)}) _{ii}   \right| > \sum_{k\neq i} \left|h^{(H)}_X(w_{(1:H)})_{ik}    \right|.
$
This means that  the matrix $[h^{(H)}_X(w_{(1:H)})_{ij} ]_{1\le i,j\le n} \in \RR^{n \times n}$
 is strictly diagonally dominant and hence is nonsingular with rank $n$.  Since the set of all columns of $h^{(H)}_X(w_{(1:H)})\in \RR^{n \times (m_H + 1)}$ contains all columns of $[h^{(H)}_X(w_{(1:H)})_{ij} ]_{1\le i,j\le n} \in \RR^{n \times n}$, this implies that $h^{(H)}_X(w_{(1:H)})$ has rank $n$ and $\varphi(w_{(1:H)})\neq 0$ for this constructed particular $w_{(1:H)}$. This proves the desired statement for the case of   $\min(m_1,\dots,m_{H-1})\ge  n$.

We now consider the remaining case of   $\min(m_1,\dots,m_{H-1})\ge  m_{x}$. Set $W^{(1)}_{1:m_x,1:m_x}=\cdots=W^{(H-1)}_{1:m_x,1:m_x}=I_{1:m_x}$,  $b^{(1)}_{1:m_x}=\alpha \mathbf{1}_{m_x}\T$, and all others to be zero. Let $\mathbf{h}(x)=h^{(H-1)}(x,w_{(1:H-1)})_{1,1:m_x}$. Then, 
$
\mathbf{h}(x)=\sigma^{\circ H-1}(x\T +\alpha\mathbf{1}_{m_x}\T), 
$
where $\sigma^{\circ l}(z)=\sigma(\sigma^{\circ l-1}(z))$ for $l \ge 1$ with         $\sigma^{\circ0}(z)=z$.
Since
$
\varsigma\sigma(z)= \ln(1 + e^{\varsigma z})= \ln((e^{-\varsigma z}+1)e^{\varsigma z})= \ln((e^{-\varsigma z}+1)e^{\varsigma z})=\ln(e^{\varsigma z})+ \ln(e^{-\varsigma z}+1)=\varsigma z+ \ln(1+e^{-\varsigma z}),
$
we have that
$
\sigma(z)=z+ \ln(1+e^{-\varsigma z})/\varsigma,
$ 
and 
\begin{align*}
\sigma^{\circ H-1}(z)&=\sigma^{\circ H-2}(z)+ \ln(1+e^{-\varsigma \sigma^{\circ H-2}(z)})/\varsigma
\\ & =z+ \sum_{l=0}^{H-2} \ln(1+e^{-\varsigma \sigma^{\circ l}(z)})/\varsigma.
\end{align*}
Therefore, 
$ 
\text{$\mathbf{h}(x)=\sigma^{\circ H-1}(x\T +\alpha\mathbf{1}_{m_x}\T) = x\T +\alpha\mathbf{1}_{m_x}\T + \psi( x\T +\alpha\mathbf{1}_{m_x}\T) $}.
$
where $\psi(z)= \sum_{l=0}^{H-2} \ln(1+e^{-\varsigma \sigma^{\circ l}(z)})/\varsigma$. Here, $\psi(z) \rightarrow 0$ as $z\rightarrow \infty$, since 
$
\sigma^{\circ l}(z)=z+ \sum_{k=0}^{l-1} \ln(1+e^{-\varsigma \sigma^{\circ k}(z)})/\varsigma \ge z 
$
and $\ln(1+e^{-\varsigma \sigma^{\circ k}(z)}) \ge 0$. From Assumption \ref{a:input}, there exists $c>0$ such that $\|x_i\|_2^2- x_i\T x_j > c$ for any $x_i,x_j \in S_{x}$ with $i\neq j$. Using such a $c>0$ as well as a $\alpha'>0$,  set $W^{(H)}_i=\alpha'x_{i}\in \RR^{m_x}$ and $b^{(H)}_i=-\alpha'\alpha\mathbf{1}_{m_x}\T x_{i}^{}+\alpha'(c/2 - \|x_i \|^2)$ for $i=1,\dots,n$. 
Set all other weights and bias to be zero. Then,
for any $i\in\{1,\dots,n\}$,
\begin{align*}
h^{(H)}_X(w_{(1:H)}) _{ii} &=\sigma(\alpha'(x\T _{i}+\alpha\mathbf{1}_{m_x}\T + \psi( x\T _{i}+\alpha\mathbf{1}_{m_x}\T))x_{i} -\alpha'\alpha\mathbf{1}_{m_x}\T x_{i}^{}+\alpha'(c/2 - \|x_i \|^2)), 
\\ & =\sigma(\alpha'(c/2+\psi( x\T _{i}+\alpha\mathbf{1}_{m_x}\T)x_{i})),
\end{align*}
and for any $j\in\{1,\dots,n\}$ with  $j \neq i$,
\begin{align*}
h^{(H)}_X(w_{(1:H)})_{ij} &=\sigma(\alpha'(x\T _{i}+\alpha\mathbf{1}_{m_x}\T + \psi( x\T _{i}+\alpha\mathbf{1}_{m_x}\T))x_{j} -\alpha'\alpha\mathbf{1}_{m_x}\T x_{j}^{}+\alpha'(c/2 - \|x_j \|^2))
\\ & =\sigma(\alpha'(x\T _{i}x_{j}-\|x_j \|^2+c/2+\psi( x\T _{i}+\alpha\mathbf{1}_{m_x}\T)x_{j})) \\ & \le \sigma(\alpha'(-c/2+\psi( x\T _{i}+\alpha\mathbf{1}_{m_x}\T)x_{j})),   
\end{align*}  
where the last inequality used the monotonicity of $\sigma$ and $\|x_i\|_2^2- x_i\T x_j > c$. Since $\sigma(\alpha'c/2)\rightarrow \infty$ and $\sigma(-\alpha'c/2) \rightarrow 0$  
as  $\alpha'\rightarrow \infty$ and  $\psi( x\T _{i}+\alpha\mathbf{1}_{m_x}\T) \rightarrow 0$ as $\alpha\rightarrow \infty$, we have that with  $\alpha,\alpha'$ sufficiently large, for any $i \in \{1,\dots,n\}$,
$
\left|h^{(H)}_X(w_{(1:H)}) _{ii}   \right| > \sum_{k\neq i} \left|h^{(H)}_X(w_{(1:H)})_{ik}    \right|,
$
which means that the matrix $[h^{(H)}_X(w_{(1:H)})_{ij} ]_{1\le i,j\le n} \in \RR^{n \times n}$
 is strictly diagonally dominant and hence is nonsingular with rank $n$.  Since the set of all columns of $h^{(H)}_X(w_{(1:H)})\in \RR^{n \times (m_H + 1)}$ contains all columns of $[h^{(H)}_X(w_{(1:H)})_{ij} ]_{1\le i,j\le n} \in \RR^{n \times n}$, this implies that $h^{(H)}_X(w_{(1:H)})$ has rank $n$ and $\varphi(w_{(1:H)})\neq 0$ for this constructed particular $w_{(1:H)}$. This proves the desired statement for the last case of   $\min(m_1,\dots,m_{H-1})\ge  m_{x}$.

\end{proof}

\section{Proof of Proposition \ref{prop:analytic_bn}}
Now we prove proposition \ref{prop:analytic_bn} to show that batch-normalization layer is real analytic.
\begin{proof}
As batch normalization works for coordinate-wise, we consider its behavior for an arbitrary coordinate: 
$$
\BN_{\gamma,\beta}(z) = \gamma \frac{z-\mu}{\sqrt{\sigma^2 + \epsilon}} + \beta 
$$
 where $z\in \RR$. Here, $\mu$ and $\sigma^2$ depend on other samples as well:
\begin{align*}
  \mu &= \frac{1}{|S|} \sum_{i\in S}  z_i \\
  \sigma^2 &= \frac{1}{|S|} \sum_{i\in S} (z_i - \mu)^2
\end{align*}
where $S$ is an arbitrary subset of $\{1,2,\dots,n\}$.  Since a composition of real analytic functions is real analytic and an affine map is real analytic,   $(z,\gamma,\beta) \mapsto \BN_{\gamma,\beta}(z)=\gamma \frac{z-\mu}{\sqrt{\sigma^2 + \epsilon}}+ \beta $ is real analytic   if $z \mapsto \frac{z-\mu}{\sqrt{\sigma^2 + \epsilon}}$ is real analytic. Since the quotient of two real analytic functions remain real analytic if the devisor is nowhere zero, $z \mapsto \frac{z-\mu}{\sqrt{\sigma^2 + \epsilon}}$ is real analytic if $z \mapsto z-\mu$ and $z \mapsto\sqrt{\sigma^2 + \epsilon}$ are real analytic. Here, $z-\mu$ is real analytic, and $z \mapsto\sqrt{\sigma^2 + \epsilon}$ is real analytic since the square root function is a real analytic function on the interval  $(0, \infty)$. Therefore,  $(z,\gamma,\beta) \mapsto \BN_{\gamma,\beta}(z)$ is real analytic. 
\end{proof}

\section{Proof of Theorem \ref{thm:all}}
The proof of Theorem \ref{thm:all} builds upon the  proofs of Theorem \ref{thm:fcnet} and \ref{thm:general}. In addition to the proofs of Theorem \ref{thm:fcnet} and \ref{thm:general}, we utilize the following two lemmas. Let $f_{i} (w)=f(x_i,w)$.

\begin{lemma} \label{lemma:pgd_approx}
Suppose Assumption \ref{a:loss} hold. Then, at any differentiable point $w\in \RR^d$ of $\Lcal$, it holds  that for any $w^*\in \RR^d$,
$$
\Lcal(w) \le \Lcal_w(w^{*})+\| \nu \odot w -w^* \| \|\nabla \Lcal(w)\|,  
$$
where 
$$
\Lcal_w(w^{*})=\frac{1}{n}\sum_{i=1}^n  \ell\left(f_{w}(x,w^{*}), y_{i}\right),
$$
and
$$
f_{w}(x,w^{*})=\sum_{k=1}^{d} w^{*}_k \frac{\partial f_{i}(w)}{\partial w_k}.
$$
\end{lemma} 
\begin{proof}
Let   $w\in \RR^{d} $ be an arbitrary differentiable  point of $\Lcal$. We first observe that   $f_{i}(w)=\sum_{k=1}^{d} (\nu \odot w)_k\frac{\partial f_{i}(w)}{\partial w_k}$ for all $i \in \{1,\dots,n\}$.  Thus, for any $w^* \in \RR^{d}$, 
\begin{align*}
 \Lcal_w(w^{*})
& \ge \frac{1}{n} \sum_{i=1}^n \left[ \ell_{i}(f_{i}(w))+ \nabla\ell_{i}(f_{i}(w))(f_{w}(x,w^{*})- f(x,w)) \right]
\\ & =\Lcal(w) - \sum_{k=1}^{d} \left((\nu \odot w)_k -w_k^* \right) \frac{1}{n}\sum_{i=1}^n  \nabla \ell_{i}(f_{i}(w))\frac{\partial f_{i}(w)}{\partial w_k}  
\\ & = \Lcal(w) - \sum_{k=1}^{d} \left((\nu \odot w)_k -w_k^* \right) \nabla_{w_k}\Lcal(w),
\\ & = \Lcal(w) -\left((\nu \odot w) -w^* \right)\T \nabla \Lcal(w)
\\ & \ge \Lcal(w) - \|(\nu \odot w) -w^* \| \|\nabla \Lcal(w)\|   
\end{align*}
where the first line follows from Assumption \ref{a:loss} (differentiable and convex  $\ell_{i}$).  
\end{proof} 

\begin{lemma} \label{lemma:known+alpha_3}
Assume that $\|\nabla\Lcal(z)-\nabla\Lcal(z')\|\le L \|z-z'\|$ for all $z,z'$ in the domain of $\Lcal$. Define the sequence $(z^{t})_{t=0}^\infty$ by $z^{t+1} = z^{t}- \frac{2\bar \eta}{L}\nabla \Lcal(z^{t})$ for any $t \ge 1$ with an initial parameter vector $z^{0}$. Then, 
$$
\min_{0\le t \le T} \|\nabla \Lcal(z^{t})\| \le\frac{1}{\sqrt{T+1}}    \sqrt{\frac{L(\Lcal(z^{0})-\Lcal(z^{*}))}{2\bar \eta(1-\bar \eta)}}. 
$$   
\end{lemma} 
\begin{proof}
Using Lemma \ref{lemma:known_1} and $z^{t+1} - z^{t}= - \frac{2\bar \eta}{L}\nabla \Lcal(z^{t})$,
\begin{align*}
\Lcal(z^{t+1}) &\le \Lcal(z^{t})  + \nabla \Lcal(z^{t})  \T (z^{t+1}-z^{t}) + \frac{L}{2} \|z^{t+1}-z^{t}\|^2
\\ & = \Lcal(z^{t})-\frac{2\bar \eta}{L} \|\nabla \Lcal(z^{t})\|^{2} +  \frac{2\bar \eta^{2}}{L}\|\nabla \Lcal(z^{t})\|^2    
\\ & = \Lcal(z^{t})-\frac{2\bar \eta}{L} (1-\bar \eta)\|\nabla \Lcal(z^{t})\|^{2},    
\end{align*} which implies that
$$
\|\nabla \Lcal(z^{t})\|^{2} \le\frac{L}{2\bar \eta(1-\bar \eta)}(\Lcal(z^{t})-\Lcal(z^{t+1})).
$$
By summing up both sides over the time,
$$
(T+1) \left(\min_{0\le t \le T} \|\nabla \Lcal(z^{t})\|^{2} \right) \le \sum_{t=0}^T \|\nabla \Lcal(z^{t})\|^{2} \le\frac{L(\Lcal(z^{0})-\Lcal(z^{T+1}))}{2\bar \eta(1-\bar \eta)}   .  
$$
Therefore,
$$
\min_{0\le t \le T} \|\nabla \Lcal(z^{t})\|\le \sqrt{\frac{L(\Lcal(z^{0})-\Lcal(z^{T+1}))}{2\bar \eta(1-\bar \eta)(T+1)}} \le\frac{1}{\sqrt{T+1}}    \sqrt{\frac{L(\Lcal(z^{0})-\Lcal(z^{*}))}{2\bar \eta(1-\bar \eta)}}.
$$

\end{proof}

The proof of Theorem \ref{thm:all} is  a result of a combination of 
these lemmas and the proofs of of Theorem \ref{thm:fcnet} and  \ref{thm:general}.
\begin{proof}[Proof of Theorem \ref{thm:all}]
From Lemma \ref{lemma:pgd_approx}, for any $w^*_k \in \RR^d$,
 $$
 \min_{\tau \le k \le t}L(w^{k}) \le\min_{\tau\le k \le t} \left(L_{w^k}(w^{*}_k)+\|(\nu \odot w^k) - w^*_k \| \|\nabla L(w^{k})\| \right).
$$
From the proofs of Theorems \ref{thm:fcnet} and \ref{thm:general}, if  $\min(m_1,\dots,m_{H-1})\ge \min(m_x, n)$ and $m_H \ge n$ or if Assumption  \ref{a:existance_of_w} hold, then we have that $\rank \left([h_{X}^{(H)}(w_{(1:H)}^{\tau}), \mathbf{1}_{n}]\right)=n$. By noticing that $[h_{X}^{(H)}(w_{(1:H)}^{\tau}), \mathbf{1}_{n}]$ appears in the block diagonal matrix $\frac{\partial \vect(f_{X}(w^{\tau}))}{\partial w_{(H+1)}}$ with $m_y$ blocks (it is block diagonal because the $k$-th output unite does not depend on the parts of last layers' weights and bias that are not connected to the $k$-th output unit), this implies that $\rank \left(\frac{\partial \vect(f_{X}(w^{\tau}))}{\partial w}\right) = nm_y$. From the assumption of $\rank \left(\mathbf{K}(w^k)\right)\ge\rank \left(\mathbf{K}(w^\tau)\right)$, this implies that $\rank \left(\frac{\partial \vect(f_{X}(w^{k}))}{\partial w}\right)=nm_y$.  Since $\rank \left(\frac{\partial \vect(f_{X}(w^{\tau}))}{\partial w}\right)=nm_{y}$, 
$$
\{\vect(f_X(w)) \in \RR^{n m_y}:w\in \RR^d \} = \left\{f_{w^k}(X,w^*_k)=\sum_{j=1}^{d} (w^*_k)_j \frac{\partial \vect(f_{X}(w^{k}))}{\partial w_j} =\frac{\partial \vect(f_{X}(w^{k}))}{\partial w}w^*_k\in \RR^{n m_y}:w^*_k\in \RR^d \right\}.
$$
This implies that for any $k$, there exists a $\hat \omega^k \in \RR^d$ such that  $\Lcal_{w^k}(\hat \omega^k)=\Lcal(w^{*})$ where $w^{*}$ is a global minimum. Thus,
 $$
 \min_{\tau\le k \le t}L(w^{k}) \le \Lcal(w^{*}) + \bar R\min_{\tau\le k \le t}  \|\nabla L(w^{k})\|.
$$
where  $\bar R = \max_{\tau\le k \le t} \min_{\hat \omega^{k}\in \bar Q_{k} }\|(\nu \odot w^{k} )- \hat \omega^{k}\|$ where $ \bar Q_k = \argmin_{\hat \omega \in \RR^d} \frac{1}{n} \sum_{i=1}^n \ell(\sum_{j=1}^{d} \hat \omega_j \frac{\partial f(x_i,w^{k})\T}{\partial w_j},y_{i})$.  From Lemma \ref{lemma:known+alpha_3},
$$
\min_{\tau\le k \le t}L(w^{k}) \le \Lcal(w^{*}) + \frac{1}{\sqrt{t-\tau+1}}    \sqrt{\frac{LR^{2}(\Lcal(w^{\tau})-\Lcal(w^{*}))}{2\bar \eta(1-\bar \eta)}} 
$$
\end{proof}

\end{document}